\documentclass[11pt]{article} %
\usepackage[margin=1in]{geometry}

\usepackage{yub}
\usepackage{smile}
\usepackage[round,compress]{natbib}
\usepackage{subcaption}
\usepackage{tcolorbox}
\usepackage{paralist}
\allowdisplaybreaks

\def\trtr{train-train~}
\def\trval{train-val~}
\def\train{{\sf train}}
\def\test{{\sf test}}

\newcommand{\bzero}{\mathbf{0}}
\ifdefined\usebigfont

\usepackage{times}
\usepackage[fontsize=13pt]{scrextend}
\usepackage[left=1.56in,right=1.56in,top=1.71in,bottom=1.77in]{geometry}
\usepackage{enumitem}

\setlist[itemize]{leftmargin=*}
\setlist[enumerate]{leftmargin=*}
\else
\fi

\def\shownotes{1}  %
\ifnum\shownotes=1
\newcommand{\authnote}[2]{$\ll$\textsf{\small #1 notes: #2}$\gg$}
\else
\newcommand{\authnote}[2]{}
\fi

\hypersetup{
    colorlinks,
    linkcolor={blue!50!black},
    citecolor={blue!50!black},
}
\colorlet{linkequation}{blue}

\title{How Important is the Train-Validation Split in Meta-Learning?}

\author{
  Yu Bai\thanks{Salesforce Research. Email: \texttt{\{yu.bai,pzhou,huan.wang,cxiong\}@salesforce.com}}
  \and
  Minshuo Chen\thanks{Georgia Tech. Email:
  \texttt{\{mchen393,tourzhao\}@gatech.edu}}
  \and
  Pan Zhou\footnotemark[1]
  \and
  Tuo Zhao\footnotemark[2]
  \and
  Jason D. Lee\thanks{Princeton University. Email:
    \texttt{jasonlee@princeton.edu}}
  \and
  Sham Kakade\thanks{University of Washington. Email: \texttt{sham@cs.washington.edu}}
  \and
  Huan Wang\footnotemark[1]
  \and
  Caiming Xiong\footnotemark[1]
}
\ifdefined\usebigfont

\usepackage{times}
\usepackage[fontsize=13pt]{scrextend}
\usepackage[left=1.56in,right=1.56in,top=1.71in,bottom=1.77in]{geometry}
\usepackage{enumitem}
\setlist[itemize]{leftmargin=*}
\setlist[enumerate]{leftmargin=*}
\else
\usepackage{enumitem}
\setlist[itemize]{leftmargin=*}
\setlist[enumerate]{leftmargin=*}
\fi

\mathchardef\mhyphen="2D
\renewcommand{\sp}{{\sf tr\mhyphen val}}
\newcommand{\nosp}{{\sf tr\mhyphen tr}}
\renewcommand{\train}{{\sf train}}
\renewcommand{\test}{{\sf test}}
\newcommand{\val}{{\sf val}}
\newcommand{\alg}{{\sf Alg}}

\usepackage{tabularx}
\usepackage{wrapfig}
\usepackage{footnote}
\usepackage{threeparttable}
\usepackage{tablefootnote}
\makesavenoteenv{tabular}
\makesavenoteenv{table}

\definecolor{C0}{HTML}{1F77B4}
\definecolor{C1}{HTML}{FF7F0E}
\definecolor{C2}{HTML}{2CA02C}
\definecolor{C3}{HTML}{D62728}
\definecolor{C4}{HTML}{9467BD}
\definecolor{C5}{HTML}{8C564B}

\begin{document}

\maketitle

\begin{abstract}
  Meta-learning aims to perform fast adaptation on a new task through learning a “prior” from multiple existing tasks. A common practice in meta-learning is to perform a train-validation split (\emph{train-val method}) where the prior adapts to the task on one split of the data, and the resulting predictor is evaluated on another split. Despite its prevalence, the importance of the train-validation split is not well understood either in theory or in practice, particularly in comparison to the more direct \emph{train-train method}, which uses all the per-task data for both training and evaluation.
  
  We provide a detailed theoretical study on whether and when the train-validation split is helpful in the linear centroid meta-learning problem. In the agnostic case, we show that the expected loss of the train-val method is minimized at the optimal prior for meta testing, and this is not the case for the train-train method in general without structural assumptions on the data. In contrast, in the realizable case where the data are generated from linear models, we show that both the train-val and train-train losses are minimized at the optimal prior in expectation. Further, perhaps surprisingly, our main result shows that the train-train method achieves a \emph{strictly better} excess loss in this realizable case, even when the regularization parameter and split ratio are optimally tuned for both methods. Our results highlight that sample splitting may not always be preferable, especially when the data is realizable by the model. We validate our theories by experimentally showing that the train-train method can indeed outperform the train-val method, on both simulations and real meta-learning tasks.
\end{abstract}

\section{Introduction}

Meta-learning, also known as ``learning to learn'', has recently emerged as a powerful paradigm for learning to adapt to unseen tasks \citep{schmidhuber1987evolutionary}. The high-level methodology in meta-learning is akin to how human beings learn new skills, which is typically done by relating to certain prior experience that makes the learning process easier.
More concretely, meta-learning does not train one model for each individual task, 
but rather learns a ``prior'' model from multiple existing tasks so that it is able to quickly adapt to unseen new tasks. Meta-learning has been successfully applied to many real problems, including few-shot image classification \citep{finn2017model,snell2017prototypical}, hyper-parameter optimization \citep{franceschi2018bilevel}, low-resource machine translation \citep{gu2018meta} and short event sequence modeling \citep{xie2019meta}.

A common practice in meta-learning algorithms is to perform a \emph{sample splitting}, where the data within each task is divided into a \emph{training split} which the prior uses to adapt to a task-specific predictor, and a \emph{validation split} on which we evaluate the performance of the task-specific predictor~\citep{nichol2018first,rajeswaran2019meta,fallah2020convergence,wang2020global}. For example, in a 5-way $k$-shot image classification task, standard meta-learning algorithms such as MAML~\citep{finn2017model} use $5k$ examples within each task as training data, %
and use additional examples (e.g. $k$ images, one for each class) as validation data. %
This sample splitting is believed to be crucial %
as it matches the evaluation criterion at meta-test time, where we perform adaptation on training data from a new task but evaluate its performance on unseen data from the same task.

Despite the aforementioned importance, performing the train-validation split has a potential drawback from the data efficiency perspective --- Because of the split, neither the training nor the evaluation stage is able to use all the available per-task data. In the few-shot image classification example, each task has a total of $6k$ examples available, but the train-validation split forces us to use these data separately in the two stages. Meanwhile, performing the train-validation split is also not the only option in practice: there exist algorithms such as Reptile~\citep{nichol2018reptile} and Meta-MinibatchProx~\citep{zhou2019efficient} that can instead use all the per-task data for training the task-specific predictor and also perform well empirically on benchmark tasks. These algorithms modify the loss function in the outer loop so that the training loss no longer matches the meta-test loss, but may have the advantage in terms of data efficiency for the overall problem of learning the best prior. So far it is theoretically unclear how these two approaches (with/without train-validation split) compare with each other, which motivates us to ask the following
\begin{quote}
  {\bf Question}: Is the train-validation split \emph{necessary} and \emph{optimal} in meta-learning?
\end{quote}

In this paper, we perform a detailed theoretical study on the importance of the train-validation split. We consider the linear centroid meta-learning problem~\citep{denevi2018learning}, where for each task we learn a linear predictor that is close to a common centroid in the inner loop, and find the best centroid in the outer loop (see Section~\ref{section:prelim} for the detailed problem setup).
We
compare two meta-learning algorithms: the {\bf train-val method} which performs the standard train-validation split, and the {\bf train-train method} which uses all the per-task data for both training and evaluation.

We summarize our contributions as follows:

\begin{compactenum}[\textbullet]
  \setlength{\itemsep}{5pt}
\item %
  We show that the train-validation split is necessary in the general agnostic setting (Section~\ref{section:agnostic}): The expected loss of the \trval method equals the meta test-time loss. In contrast, the \trtr method has a different expected loss and is not minimized at the best test-time centroid in general, for which we construct a concrete counter-example. %

\item In the perhaps more interesting realizable setting, we show the train-validation split is not necessary: When the tasks are generated from noiseless linear models, the expected loss of both the \trval and \trtr methods are minimized at the best test-time centroid (Section~\ref{section:realizable-population}).
  
\item Our main theoretical contribution shows that {\bf the train-validation split is non-optimal} in the realizable setting: The MSE (and test loss) of the two methods concentrates sharply around $C^{\set{\sp, \nosp}}/T$ when $T$ (the number of tasks) is large, where the constants depend on the \{dimension, per-task sample size, regularization parameter\}. A precise comparison of constants further shows that $C^{\nosp}<C^{\sp}$ when we optimally tune the regularization parameter in both methods (Section~\ref{section:realizable-mse}). Thus, in the realizable setting, the \trtr method performs strictly better than the \trval method, which is in stark contrast with the agnostic case. This result provides a novel insight into the effect of the train-validation split on the sample complexity of meta-learning.

\item We perform meta-learning experiments on simulations and benchmark few-shot image classification tasks, showing that the \trtr method consistently outperforms the \trval method (Section~\ref{appendix:experiments} \& Appendix~\ref{appendix:deep}). This validates our theories and presents empirical evidence that sample-splitting may not be crucial; methods that utilize the per-task data more efficiently may be preferred.

\item On the technical end, our main results in Section~\ref{section:realizable} build on concentration analyses on a group of ridge-covariance matrices, as well as tools from random matrix theory in the proportional regime, which may be of broader interest. (See Section~\ref{section:overview-techniques} for an overview of techniques.)
\end{compactenum}

\ifdefined\minusspace\vspace{-0.5em}\else\fi
\subsection{Related work}
\ifdefined\minusspace\vspace{-0.5em}\else\fi

\paragraph{Meta-learning and representation learning theory} \citet{baxter2000model} provided the first theoretical analysis of meta-learning via covering numbers, and \citet{maurer2016benefit} improved the analysis via Gaussian complexity techniques. Another recent line of theoretical work analyzed gradient-based meta-learning methods~\citep{denevi2018incremental,finn2019online,khodak2019adaptive,ji2020convergence} and showed guarantees for convex losses by using tools from online convex optimization. \citet{saunshi2020sample} proved the success of Reptile in a one-dimensional subspace setting. \citet{wang2020guarantees} compared the performance of \trtr and \trval methods for learning the learning rate. \citet{denevi2018learning} proposed the linear centroid model studied in this paper, and provided generalization error bounds for \trval method; the bounds proved also hold for \trtr method, so are not sharp enough to compare the two algorithms. \citet{wang2020globalmaml, wang2020global} studied the convergence of gradient-based meta-learning by relating to the kernelized approximation. \citet{arnold2019maml} observe that MAML adapts better with a deep model architecture both empirically and theoretically.

On the representation learning end,~\citet{du2020few, tripuraneni2020provable, tripuraneni2020theory} showed that ERM can successfully pool data across tasks to learn the representation. Yet the focus is on the accurate estimation of the common representation, not on the fast adaptation of the learned prior. Several recent work compares MAML versus ERM style approches~\citep{gao2020modeling,collins2020does}; these comparisons couple the effect of sample splitting with other factors such as whether the algorithm uses per-task adaptation.
Lastly, we remark that there are analyses for other representation learning schemes~\citep{mcnamara2017risk,galanti2016theoretical,alquier2016regret}.

\paragraph{Empirical understandings of meta-learning}
\citet{raghu2020rapid}
showed that MAML with a full finetuning inner loop mostly learns the top-layer linear classifier and does not change the representation layers much. This result partly justifies the validity of our linear centroid meta-learning problem in which the features (representations) are fixed and only a linear classifier is learned.~\citet{goldblum2020unraveling} investigated the difference of the neural representations learned by classical training (supervised learning) and meta-learning, and showed that the meta-learned representation is better for downstream adaptation and makes classes more separated. %
Additionally,~\citet{setlur2020support, yao2020don} investigated alternative ways of choosing the support set (training split) in meta-learning.

\paragraph{Multi-task learning} Multi-task learning also exploits structures and similarities across multiple tasks. The earliest idea dates back to \citet{caruana1997,thrun1998,baxter2000model}, initially in connections to neural network models. They further motivated other approaches using kernel methods \citep{evgeniou2005learning,argyriou2007multi} and multivariate linear regression models with structured sparsity \citep{liu2009blockwise,liu2015calibrated}. More recent advances on deep multi-task learning focus on learning shared intermediate representations across tasks \cite{ruder2017overview}. These multi-task learning approaches usually minimize the joint empirical risk over all tasks, and the models for different tasks are enforced to share a large amount of parameters. In contrast, meta-learning only requires the models to share the same ``prior'', and is more flexible than multi-task learning.

\ifdefined\minusspace\vspace{-0.5em}\else\fi
\section{Preliminaries}
\label{section:prelim}
\ifdefined\minusspace\vspace{-0.5em}\else\fi
In this paper, we consider the standard meta-learning setting, in which we observe data from $T\ge 1$ supervised learning tasks, and the goal is to find a prior (or ``initialization'') using the combined data, such that the $(T+1)$-th new task may be solved sample-efficiently using the prior.

\paragraph{Linear centroid meta-learning}
We instantiate our study on the \emph{linear centroid meta-learning problem} (also known as learning to learn around a common mean, \citet{denevi2018learning}), where we wish to learn a task-specific linear predictor $\wb_t\in\R^d$ in the inner loop for each task $t$, and learn a ``centroid'' $\wb_0$ in the outer loop that enables fast adaptation to $\wb_t$ within each task:
\begin{quote}
  Find the best centroid $\wb_0\in\R^d$ for adapting to a linear predictor $\wb_t$ on each task $t$.
\end{quote}

Formally, we assume that we observe training data from $T\ge 1$ tasks, where for each task index $t$, we sample a task $p_t$ (a distribution over $\R^d\times \R$) from some distribution of tasks $\Pi$, and observe $n$ examples $(\Xb_t, \yb_t)\in\R^{n\times d}\times \R^n$ that are drawn i.i.d. from $p_t$:
\begin{align}
  \label{equation:general-distribution}
  p_t\sim \Pi,~(\Xb_t, \yb_t) = \set{(\xb_{t,i}, y_{t,i})}_{i=1}^n~\text{where}~(\xb_{t,i}, y_{t,i})\simiid p_t.
\end{align}
We do not make further assumptions on $(n,d)$; in particular, we allow the underdetermined setting $n\le d$, in which there exists (one or many) interpolators $\wt{\wb}_t$ that perfectly fit the data: $\Xb_t\wt{\wb}_t = \yb_t$.

\paragraph{Inner loop: Ridge solver with biased regularization towards the centroid}
Our goal in the inner loop is to find a linear predictor $\wb_t$ that fits the data in task $t$ while being close to the given ``centroid'' $\wb_0\in\R^d$. We instantiate this through ridge regression (i.e. linear regression with $L_2$ regularization) where the regularization biases $\wb_t$ towards the centroid. Formally, for any $\wb_0\in\R^d$ and any dataset $(\Xb, \yb)$, we consider the algorithm
\begin{align*}
  & \quad \cA_\lambda(\wb_0; \Xb, \yb) \defeq \argmin_{\wb} \frac{1}{n}\norm{\Xb\wb - \yb}^2 + \lambda\norm{\wb - \wb_0}^2  \\
  & = \wb_0 + \paren{ \Xb^\top\Xb + n\lambda \Ib_d}^{-1}\Xb^\top(\yb - \Xb\wb_0),
\end{align*}
where $\lambda>0$ is the regularization strength (typically a tunable hyper-parameter). As we regularize by $\norm{\wb-\wb_0}^2$, this inner solver encourages the solution to be close to $\wb_0$, as we desire. Such a regularizer is widely used in practical meta-learning algorithms such as MetaOptNet~\citep{lee2019meta} and Meta-MinibatchProx~\citep{zhou2019efficient}. In addition, as $\lambda\to 0$, this solver recovers gradient descent fine-tuning: we have
\begin{align*}
  & \quad \cA_0(\wb_0;\Xb, \yb) \defeq \lim_{\lambda\to 0} \cA_\lambda(\wb_0; \Xb, \yb) \\
  & = \wb_0 + \Xb^\dagger (\yb-\Xb\wb_0) = \argmin\nolimits_{\Xb\wb=\yb} \norm{\wb - \wb_0}^2,
\end{align*}
where $\Xb^{\dagger}\in\R^{d\times n}$ denotes the pseudo-inverse of $\Xb$.
This is the \emph{minimum-distance} interpolator of $(\Xb, \yb)$ and also the solution found by gradient descent \footnote{with a small step-size, or gradient flow.} on $\norm{\Xb\wb - \yb}^2$ initialized at $\wb_0$. Therefore our ridge solver with $\lambda>0$ can be seen as a generalized version of the gradient descent solver used in MAML~\citep{finn2017model}.

\paragraph{Outer loop: Learning the best centroid}
In the outer loop, our goal is to find the best centroid $\wb_0$. The standard approach in meta-learning is to perform a \emph{train-validation split}, that is, (1) execute the inner solver on a first split of the task-specific data, and (2) evaluate the loss on a second split, yielding a function of $\wb_0$ that we can optimize. 
This two-stage procedure can be written as
\begin{align*}
  & \textrm{Compute}~\wb_t(\wb_0) = \cA_\lambda(\wb_0; \Xb_t^{\train}, \yb_t^{\train})~~~\textrm{and}~~~\textrm{evaluate}~\norm{\yb_t^{\val} - \Xb_t^{\val}\wb_t(\wb_0)}^2.
\end{align*}
where $(\Xb_t^{\train}, \yb_t^{\train})=\set{(\xb_{t,i}, y_{t,i})}_{i=1}^{n_1}$ and $(\Xb_t^{\val}, \yb_t^{\val})=\set{(\xb_{t,i}, y_{t,i})}_{i=n_1+1}^{n}$ are two disjoint splits of the per-task data $(\Xb_t, \yb_t)$ of size $(n_1, n_2)$, with $n_1+n_2=n$. This amounts to the
\begin{tcolorbox}
  {\bf Train-val method}: Output $\hat{\wb}_{0,T}^{\sp}$ that minimizes
  \begin{equation}
    \label{equation:split-loss}
    \begin{aligned}
      \hat{L}^{\sp}_T(\wb_0) = \frac{1}{T}\sum_{t=1}^T \ell^{\sp}_{t}(\wb_0) \defeq \frac{1}{T}\sum_{t=1}^T \frac{1}{2n_2} \norm{ \yb_t^{\val} - \Xb_t^{\val} \cA_\lambda(\wb_0; \Xb_t^{\train}, \yb_t^{\train})}^2.
    \end{aligned}
  \end{equation}
\end{tcolorbox}

We compare the \trval method to an alternative version, where we do not perform the train-validation split, but instead use \emph{all the per-task data for both training and evaluation}. Formally, this is to consider the
\begin{tcolorbox}
  {\bf Train-train method}: Output $\hat{\wb}_{0,T}^{\nosp}$ that minimizes
  \begin{equation}
    \label{equation:non-split-loss}
    \begin{aligned}
      \hat{L}^{\nosp}_T(\wb_0) = \frac{1}{T}\sum_{t=1}^T \ell^{\nosp}_{t}(\wb_0) = \frac{1}{T}\sum_{t=1}^T \frac{1}{2n} \norm{\yb_t - \Xb_t\cA_\lambda(\wb_0; \Xb_t, \yb_t)}^2.
    \end{aligned}
  \end{equation}
\end{tcolorbox}

Let $L^{\set{\sp, \nosp}}(\wb_0)= \E\brac{\ell_t^{\set{\sp, \nosp}}(\wb_0)}$ denote the corresponding expected losses. We remark that this expectation is equivalent to observing an infinite amount of tasks, but still with a finite $(n,d)$ within each task.

\paragraph{(Meta-)Test time} The meta-test time performance of any meta-learning algorithm is a joint function of the (learned) centroid $\wb_0$ and the inner algorithm $\alg$. Upon receiving a new task $p_{T+1}\sim \Pi$ and training data $(\Xb_{T+1}, \yb_{T+1})\in\R^{n\times d}\times \R^{n}$, we run the inner loop $\alg$ with prior $\wb_0$ on the training data, and evaluate it on an (unseen) test example $(\xb', y')\sim p_{T+1}$:
\begin{align*}
  L^{\test}(\wb_0; \alg) \defeq  \E\brac{
  \frac{1}{2}\paren{ \xb'^\top \alg(\wb_0; \Xb_{T+1}, \yb_{T+1})- y' }^2 }.
\end{align*}
Additionally, for both \trval and \trtr methods, we need to ensure that the inner loop used for meta-test is exactly the same as that used in meta-training. Therefore, the meta-test performance for the \trval and \trtr methods above should be evaluated as
\begin{align*}
  & L^{\test}_{\lambda, n_1}(\hat{\wb}_{0,T}^{\sp}) \defeq L^{\test}(\hat{\wb}_{0,T}^{\sp}; \cA_{\lambda, n_1}), \\
  & L^{\test}_{\lambda, n}(\hat{\wb}_{0,T}^{\nosp}) \defeq L^{\test}(\hat{\wb}_{0,T}^{\nosp}; \cA_{\lambda, n}),
\end{align*}
where $\cA_{\lambda, m}$ denotes the ridge solver with regularization strength $\lambda>0$ on $m\le n$ data points. Finally, we let
\begin{align}
  \label{equation:best-test-w}
  \wb_{0,\star}(\lambda; n) = \argmin_{\wb_0} L^{\test}_{\lambda, n}(\wb_0)
\end{align}
denote the best centroid if the inner loop uses $\cA_{\lambda, n}$. The performance of the \trval algorithm $\hat{\wb}_{0,T}^{\sp}$ should be compared against $\wb_{0,\star}(\lambda, n_1)$, whereas the \trtr algorithm $\hat{\wb}_{0,T}^{\nosp}$ should be compared against $\wb_{0,\star}(\lambda, n)$.

\section{The importance of sample splitting}
\label{section:agnostic}
We begin by analyzing the \trtr and \trval methods defined in~\eqref{equation:split-loss} and~\eqref{equation:non-split-loss}, in the agnostic setting where we do not make structural assumptions on the data distribution $p_t$.

In this case, we show that the importance of the sample splitting is clear even at the population level: the expected loss of the \trval method matches the test-time loss, whereas the expected loss of the \trtr method does not match the test-time in general and have a different minimizer.

\begin{theorem}[Properties of expected losses in the agnostic case]
  \label{theorem:agnostic}
  Suppose the task distributions satisfy $\E_{\xb\sim p_t}[\xb\xb^\top] \succ \bzero$, $\E_{\xb\sim p_t}[\norm{\xb}^4]<\infty$ and $\E_{(\xb, y)\sim p_t}[\norm{\xb y}]<\infty$ for almost surely all $p_t\sim \Pi$, but can be otherwise arbitrary. Then, we have the following:
  \begin{enumerate}[label=(\alph*)]
  \item (Unbiased loss for \trval method) For any $\lambda>0$ and any $(n_1, n_2)$ such that $n_1+n_2=n$, the expected loss of the \trval method is equal to the meta test-time loss, and thus minimized at the best test-time centroid:
    \begin{align*}
      L^{\sp}_{\lambda, n_1, n_2}(\wb_0) = L^{\test}_{\lambda, n_1}(\wb_0).
    \end{align*}
  \item (Biased loss for \trtr method) There exists a distribution of tasks $\Pi$ on $d=1$ satisfying the above conditions, on which for any $n\ge 1$ and $\lambda>0$, the expected loss of the \trtr method is not equal to the test-time loss, and the minimizers are not equal: 
  \begin{align*}
    & L^{\nosp}_{\lambda, n}(\cdot) \neq L^{\test}_{\lambda, n}(\cdot),~~~{\rm and} \\
    & \wb_{0,\star}^{\nosp} \defeq \argmin_{\wb_0} L^{\nosp}(\wb_0) \neq \argmin_{\wb_0} L^{\test}_{\lambda, n}(\wb_0).
  \end{align*}
  Further, the excess test loss of $\wb_{0,\star}^{\nosp}$ is bounded away from zero: $L^{\test}_{\lambda, n}(\wb_{0,\star}^{\nosp}) - \min_{\wb_0} L^{\test}_{\lambda, n}(\wb_0) > 0$.
  \end{enumerate}
\end{theorem}

Theorem~\ref{theorem:agnostic} makes clear the advantage of the \trval method when there is no structural assumption on the data distributions: The expected version of the \trval loss matches the meta test-time, whereas the expected version of the \trtr loss has a bias in general. By standard consistency results~\citep{van2000asymptotic}, this advantage carries on to the sampled versions as well for large $T$. In other words, the \trval method is a ``valid ERM'' (empirical risk minimization) procedure for the test-time loss, whereas the \trtr method is not a valid ERM. 

\paragraph{Proof intuitions}
The proof of part (a) follows from direct calculations, whereas the proof of part (b) is trickier as we need to construct a counter-example in which the expected loss of the \trtr method is not equal the test-time loss for \emph{any} $\lambda,n$. We provide such a construction in $d=1$, where the distribution $p_t$ has a certain asymmetry that results in a bias the \trtr loss function for any $\lambda$ and $n$. However, we expect such a bias to be present in general for any dimensions. The proof of Theorem~\ref{theorem:agnostic} can be found in Appendix~\ref{appendix:proof-agnostic}.

\section{Is sample splitting always optimal?}
\label{section:realizable}
Theorem~\ref{theorem:agnostic} states a negative result for the \trtr method, showing that its expected loss and the meta test-time loss does not have the same values and minimizers. However, such a result does not preclude the possibility that there exists a data distribution on which the minimizers coincide (even though the loss values can still be different).

In this section, we construct a simple data distribution on which \trtr method is indeed unbiased in terms of the minimizer of the expected loss, and compare its performance against the \trval method more explicitly.

\paragraph{Realizable linear model}
We consider the following instantiation of the (generic) meta-learning data distribution assumption in~\eqref{equation:general-distribution}: We assume that each task $p_t$ is specified by a $\wb_t\in\R^d$ sampled from some distribution $\Pi$ (overloading notation), and the observed data follows the noiseless linear model with ground truth parameter $\wb_t$:
\begin{equation}
  \label{equation:realizable-model}
  \yb_t = \Xb_t \wb_t.
\end{equation}
Note that when $n\ge d$ and inputs are in general position, we are able to perfectly recover $\wb_t$ (by solving linear equations), therefore the problem in the inner loop is easy. However, even in this case the outer loop problem is still non-trivial as we wish to learn the best centroid $\wb_0$.

\subsection{Population minimizers}
\label{section:realizable-population}

We first show that on the realizable linear model~\eqref{equation:realizable-model}, the test-time best centroids $\wb_{0,\star}(\lambda, n)=\argmin_{\wb_0}L^{\test}_{\lambda, n}(\wb_0)$ is the same for any $(\lambda, n)$, and both the \trtr and \trval methods are unbiased: Both expected losses are minimized at $\wb_{0,\star}$.

\begin{theorem}[Population minimizers on the realizable model]
  \label{theorem:consistency-realizable-model}
  On the realizable linear model~\eqref{equation:realizable-model}, suppose $\E[\norm{\xb}^4]<\infty$ and $\E[\norm{\wb_t}^2]<\infty$. Then the test-time meta loss for all $\lambda>0$ and all $n$ is minimized at the same point, that is, the mean of the ground truth parameters:
  \begin{align*}
    & \quad \wb_{0,\star}(\lambda, n) = \argmin_{\wb_0} L^{\test}_{\lambda, n}(\wb_0) \\
    & = \wb_{0,\star} \defeq \E_{\wb_t\sim\Pi}[\wb_t],~~~\textrm{for all}~\lambda>0,~n.
  \end{align*}
  Furthermore, for both the \trval method and the \trtr method, the expected loss is minimized at $\wb_{0,\star}$ for any $\lambda>0$, $n$, and $(n_1,n_2)$:
  \begin{align*}
    \argmin_{\wb_0} L^{\sp}_{\lambda, n_1, n_2}(\wb_0) = \argmin_{\wb_0} L^{\nosp}_{\lambda, n}(\wb_0) = \wb_{0,\star}.
  \end{align*}
\end{theorem}
Theorem~\ref{theorem:consistency-realizable-model} shows that both the \trval and \trtr methods are in expectation minimized at the same optimal parameter $\wb_{0,\star}$ which is the mean of $\wb_t$. This is a consequence of the good structure in our realizable linear model~\eqref{equation:realizable-model}: at a high level, $\wb_{0,\star}$ is indeed the best centroid since it has (on average) the closest distance to a randomly sampled $\wb_t$. The proof of Theorem~\ref{theorem:consistency-realizable-model} be found in Appendix~\ref{appendix:proof-consistency-realizable-model}.

\subsection{Precise comparison of rates}
\label{section:realizable-mse}

Theorem~\ref{theorem:consistency-realizable-model} suggests that we are now able to compare performance of the two methods based on their parameter estimation error (for estimating $\wb_{0,\star}$). Towards a fine-grained comparison between the \trtr and \trval methods, we make the following assumption on the distributions of $\Xb_t$ and $\wb_t$:
\begin{assumption}[Data distributions for realizable linear model]
  \label{assumption:realizable}
  The inputs are standard Gaussian: $\xb_{t, i}\simiid \normal(\bzero, \Ib_d)$. The true coefficient $\wb_t$ is independent of $\Xb_t$ and satisfies
  \begin{align}
    \label{equation:cov-wt}
    {\rm Cov}(\wb_t) = \E_{\wb_t} \brac{ (\wb_t - \wb_{0,\star})(\wb_t - \wb_{0,\star})^\top } = \frac{R^2}{d}\Ib_d,
  \end{align}
  for some fixed $R^2>0$, and that the individual entries $\set{w_{t,i} - w_{0,\star,i}}_{i\in[d], t\in[T]}$ are i.i.d. mean-zero and $KR^2/d$-sub-Gaussian for some absolute constant $K=O(1)$.
\end{assumption}
The sub-Gaussian assumption on $\wb_t$ allows for a sharp concentration of the MSE to its expectation (over $\wb_t$). The Gaussian input assumption allows for a precise characterization of certain ridge covariance type random matrices.

We are now ready to state our two main theorems, which provide a precise comparison of the MSEs of the \trtr and \trval methods under the realizable linear model.

\begin{theorem}[Concentration of MSEs in the realizable linear model]
  \label{theorem:concentration-mse}
    In the realizable linear model~\eqref{equation:realizable-model}, suppose Assumption~\ref{assumption:realizable} holds, $T=\wt{\Omega}(d)$, $d/n=\Theta(1)$, $n_2/n=\Theta(1)$, and $\lambda=\Theta(1)>0$. Then with probability at least $1-Td^{-10}$, the MSE of the \trtr and \trval methods has the following concentrations, respectively:
  \begin{align*}
    &  \norm{\hat{\wb}_0^\nosp - \wb_{0,\star}}^2 = \frac{R^2}{T} \paren{ C^{\nosp}_{d, n, \lambda} + \wt{O}\paren{ \sqrt{\frac{d}{T} } +  \frac{1}{\sqrt{d}} } }, \\
    & \norm{\hat{\wb}_0^\sp - \wb_{0,\star}}^2 = \frac{R^2}{T} \paren{C^{\sp}_{d, n_1, n_2, \lambda} + \wt{O}\paren{ \sqrt{\frac{d}{T}}  + \frac{1}{\sqrt{d}}  }},
  \end{align*}
  where $\wt{O}(\cdot)$ hides $\log(ndT)$ factor. Further, the constants $C^{\nosp},C^{\sp}=\Theta(1)$ and have explicit expressions:
  \begin{align*}
    & C^{\nosp}_{d, n, \lambda}  = \frac{ \frac{1}{d}\E\brac{\tr\paren{ (\hat{\bSigma}_n + \lambda\Ib_d)^{-4}\hat{\bSigma}_n^2 }} }{ \paren{ \frac{1}{d}\E\brac{\tr\paren{ (\hat{\bSigma}_n + \lambda\Ib_d)^{-2}\hat{\bSigma}_n }} }^2 }, \\
    & C^{\sp}_{d,n_1, n_2, \lambda} = 
      \frac{ \frac{1}{dn_2}\E\brac{ \tr\paren{ (\hat{\bSigma}_{n_1} + \lambda\Ib_d)^{-2} }^2 + (n_2+1) \tr\paren{ (\hat{\bSigma}_{n_1} + \lambda\Ib_d)^{-4} } }}{ \paren{\frac{1}{d} \E\brac{ \tr\paren{ (\hat{\bSigma}_{n_1} + \lambda\Ib_d)^{-2} } }}^2}
      ,
  \end{align*}
  where $\hat{\bSigma}_n\defeq \Xb_t^\top\Xb_t/n$ denotes the empirical covariance of a standard Gaussian random matrix $\Xb_t\in\R^{n\times d}$.
\end{theorem}
Theorem~\ref{theorem:concentration-mse} asserts that the MSEs of both methods concentrate around $R^2/T$ times a $\Theta(1)$ constant, when both $T,d$ are large and $T=\wt{\Omega}(d)$ (so that the error terms vanish). This allows us to compare the performances of the \trtr and \trval methods based on the constants. For a fair comparison, we look at the constants with optimal choices of $\lambda$ and the split ratio, which we state in the following
\begin{theorem}[Comparison of constants $C^{\nosp}$ and $C^{\sp}$]
  \label{theorem:comparison-mse}
  In the high-dimensional limiting regime $d,n\to\infty$, $d/n\to\gamma\in(0,\infty)$, the optimal constant of the \trtr method obtained by tuning the regularization $\lambda \in(0, \infty)$ satisfies
  \begin{align*}
    \inf_{\lambda>0} \lim_{d,n\to\infty,d/n\to\gamma} C^{\nosp}_{d, n, \lambda} = \inf_{\lambda>0} \rho_{\lambda,\gamma} \stackrel{(\star)}{\le} \max\set{1 + \frac{5}{27}\gamma, \frac{5}{27}+\gamma},
  \end{align*}
  where
  $\rho_{\lambda,\gamma} \defeq  4\gamma^2 \brac{(\gamma-1)^2 + (\gamma+1)\lambda}/(\lambda +\gamma+1 - \sqrt{(\lambda+\gamma+1)^2-4\gamma})^2 / \paren{(\lambda+\gamma+1)^2-4\gamma}^{3/2} $,
  and the inequality becomes equality at $\gamma=1$. In contrast, the optimal rate of the \trval method by tuning the regularization $\lambda\in(0,\infty)$ and split ratio $s\in(0,1)$ is
  \begin{align*}
    & \inf_{\lambda>0,s\in(0,1)} \lim_{d,n\to\infty,d/n\to\gamma} C^{\sp}_{d,ns, n(1-s),\lambda}
     = (1+\gamma)R^2.
  \end{align*}
As $\max\set{1+5\gamma/27, 5/27+\gamma}<1+\gamma$ for any $\gamma>0$, the \trtr method has a strictly better constant than the \trval method when $\lambda$ and $s$ are optimally tuned in both methods.
\end{theorem}

\paragraph{Implications} 
Theorem \ref{theorem:comparison-mse} shows that, perhaps surprisingly, the train-train method achieves a strictly better MSE (in terms of the constant) than the train-val method in the  realizable linear model\footnote{The same conclusion also holds for the excess test loss, as the Hessian of the test loss is a rescaled identity, see Appendix~\ref{appendix:moments}.}.
(See Figure~\ref{fig:simulationappendix}(a) for a visualization of the exact optimal rates and the upper bound $(\star)$.) This suggests that the train-validation split may not be crucial when the data has structural assumptions such as realizability by the model. To the best of our knowledge, this is the first theoretical result that offers a disentangled comparison of meta-learning algorithms with and without sample splitting. Note that our result features an \emph{optimal tuning of hyperparameters}: we compare the rates at the (theoretically) optimal $\lambda$ for the \trtr method and the optimal $\lambda,n_1$ for the \trval method.

We also remark that, while our theory considers the linear centroid meta-learning problem, our real data experiments in Section~\ref{section:deep} suggests that the superiority of the \trtr method may also hold on real meta-learning tasks with neural networks.

\begin{figure*}[!t]
\centering
\includegraphics[width = \textwidth]{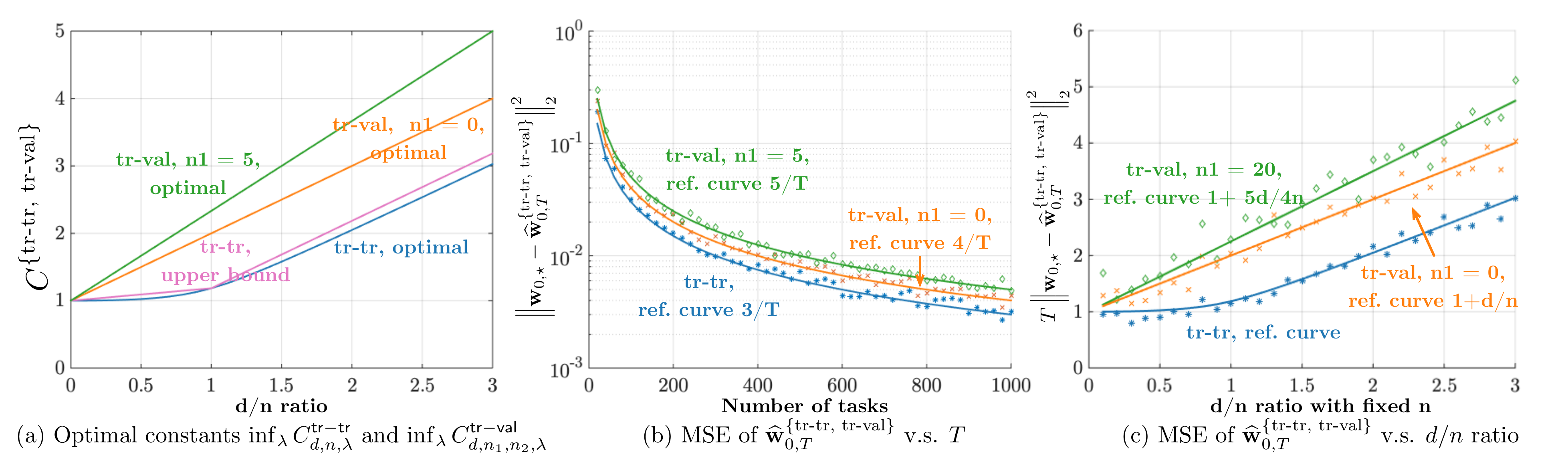}
\caption{\small Panel (a) plots the exact constants in Theorem~\ref{theorem:comparison-mse}: The optimal \trtr constant $\inf_{\lambda} C^{\nosp}_{d,n,\lambda}$ (\textcolor{C0}{blue}) and its upper bound ($\star$) (\textcolor{magenta}{magenta}), as well as the optimal \trval constant $\inf_{\lambda} C^{\sp}_{d,n_1, n_2, \lambda}$ with $n_1 = 0$ (\textcolor{C1}{orange}, optimal choice) and $n_1 = 5$ (\textcolor{C2}{green}). (Optimal $\inf_{\lambda}C^{\sp}_{d,n_1,n_2,\lambda}$ at each $n_1$ can be found in Lemma~\ref{lemma:optimally-tuned-rates}.) Curves in panel (a) are used as reference curves in plots (b) and (c). Panel (b) plots the MSE of $\hat{\wb}_{0, T}^{\{\nosp, \sp\}}$ as the total number of tasks increases from $20$ to $1000$ with an increment of $20$. We fix data dimension $d = 60$ and per-task sample size $n = 20$. For the \trval method, we experiment on $n_1 = 0$ and $n_1 = 5$. Panel (c) shows the rescaled MSE of $\hat{\wb}_{0, T}^{\{\nosp, \sp\}}$ as the ratio $d / n$ varies from $0$ to $3$ (with $n=100$ and $T = 300$).
}
\label{fig:simulationappendix}
\end{figure*}

\subsection{Overview of techniques}
\label{section:overview-techniques}

Here we provide an overview of the techniques in proving Theorem~\ref{theorem:concentration-mse} and Theorem~\ref{theorem:comparison-mse}. We defer the full proofs to Appendix~\ref{appendix:proof-concentration-mse} and Appendix~\ref{appendix:proof-comparison-mse} respectively.

\paragraph{Closed-form expressions for $\hat{\wb}_{0,T}^{\nosp}$ and $\hat{\wb}_{0,T}^{\sp}$}
Our first step is to obtain the following closed-form expressions for the estimation errors of both methods in the realizable linear model (see Lemma~\ref{lemma:closed-form-estimators}):
\begin{align*}
  & \hat{\wb}^{\nosp}_{0,T} - \wb_{0,\star} = \paren{\sum_{t=1}^T \Ab_t}^{-1} \sum_{t=1}^T \Ab_t(\wb_t - \wb_{0,\star}), \\
  & \hat{\wb}^{\sp}_{0,T} -\wb_{0,\star} = \paren{\sum_{t=1}^T \Bb_t}^{-1} \sum_{t=1}^T \Bb_t (\wb_t - \wb_{0,\star}), 
\end{align*}
where
\begin{align*}
  & \Ab_t \defeq \lambda^2 \paren{ \Xb_t^\top\Xb_t/n+ \lambda\Ib_d}^{-2} (\Xb_t^\top\Xb_t/n),\\
  & \Bb_t \defeq \lambda^2 \paren{ \Xb_t^{\train\top}\Xb_t^{\train}/n_1 + \lambda\Ib_d}^{-1} \paren{\Xb_t^{\val\top}\Xb_t^{\val} / n_2} \cdot \paren{ \Xb_t^{\train\top}\Xb_t^{\train} / n_1 + \lambda\Ib_d}^{-1}.
\end{align*}
These expressions simplify the estimation errors as the ``weighted averages'' of the $\set{\wb_t - \wb_0}$ with weighting matrices $\Ab_t$ and $\Bb_t$.

\paragraph{Sharp concentration to exact constants}
Our next step is to establish the concentration
\begin{align*}
  & \quad \norm{\hat{\wb}_{0,T}^{\nosp} - \wb_{0,\star}}^2 \stackrel{(i)}{\approx} \frac{R^2}{d} \cdot \tr\paren{ \paren{\sum_{t=1}^T \Ab_t}^{-2} \paren{\sum_{t=1}^T \Ab_t^2}} \\
  & \stackrel{(ii)}{\approx} \frac{R^2}{T} \paren{ \tr\paren{ \E[\Ab_t] } / d}^{-2} \paren{ \tr\paren{ \E[\Ab_t^2]} / d} = \frac{R^2}{T} C^{\nosp}_{d,n,\lambda}.
\end{align*}
(and a similar result for $\hat{\wb}_{0,T}^{\sp}$ using $\Bb_t$.)
Above, $(i)$ relies on the concentration of a certain quadratic form involving the $(\wb_t - \wb_0)$'s, following from the Hanson-Wright inequality (cf. Lemma~\ref{lemma:hanson-wright}), and $(ii)$ relies on the concentration of the matrices $\sum_{t=1}^T\Ab_t/T$ and $\sum_{t=1}^T\Ab_t^2/T$, using standard sub-Gaussian matrix concentration and a truncation argument (cf. Lemma~\ref{lemma:concentration-atbt}). Further calculating the expectations $\E[\Ab_t]$ and $\E[\Ab_t^2]$ gives the exact formula of $C^{\nosp}_{d,n,\lambda}$ (cf. Lemma~\ref{lemma:moments}) and finishes the proof of Theorem~\ref{theorem:concentration-mse}.

\paragraph{Optimizing and comparing $C^{\nosp}_{d,n,\lambda}$ and $C^{\sp}_{d,n_1,n_2,\lambda}$}

The constants $C^{\nosp}_{d,n,\lambda}$ and $C^{\sp}_{d,n_1,n_2,\lambda}$ involve tunable hyperparameters $\lambda$ (for both methods) and $n_1$ (for the \trval method). We use the following strategies to optimize the hyperparameters in each method, which combine to yield Theorem~\ref{theorem:comparison-mse}.

\begin{itemize}
\item For the \trval method, we show that the optimal tunable parameters for any $(n,d)$ is taken at a special case $\lambda=\infty$ and $(n_1,n_2)=(0,n)$, at which the rate only depends on $\frac{1}{n_1}\Xb_t^{\train \top}\Xb_t^{\train}$ through its rank (and thus has a simple closed-form). We state this result in Lemma~\ref{lemma:optimally-tuned-rates}. The proof builds on algebraic manipulations of the quantity $C^{\sp}_{d,n,\lambda_1,\lambda_2}$, and can be found in Appendix~\ref{appendix:optimal-hyperparam-trval}.
\item For the \trtr method, we apply random matrix theory to simplify the spectrum of $\frac{1}{n}\Xb_t^\top\Xb_t$ in the \emph{proportional limit} where $d,n\to\infty$ and $d/n$ stays as a constant ~\citep{bai2010spectral,anderson2010introduction}, and obtain a closed-form expression of the asymptotic MSE for any $\lambda>0$, which we can analytically optimize over $\lambda$. We state this result in Theorem~\ref{theorem:high-dim-limit}. The proof builds on the Stieltjes transform and its ``derivative trick''~\citep{dobriban2018high}, and is deferred to Appendix~\ref{appendix:optimal-hyperparam-trtr}.
\end{itemize}

	\begin{table*}[!t]
          \begin{threeparttable}[t]
			\caption{Comparison of \trtr and \trval on few-shot image classification (accuracy in $\%$).
                        }
			\setlength{\tabcolsep}{20pt}  
			\renewcommand{\arraystretch}{1.5} 
			\label{comparisontable} \begin{center}
				\footnotesize {
						\begin{tabularx}{\textwidth}{c|c|cc cc}	\bottomrule
							\multirow{3}{*}{\vspace{-0.7em}\rotatebox{90}{miniImage}}  & method & 1-shot 5-way & 5-shot 5-way& 1-shot 20-way & 5-shot 20-way  \\
							\cline{2-6}
							
							&	 \trval &48.76 $\pm$ 0.87&	63.56 $\pm$ 0.95 &17.52 $\pm$ 0.49& 21.32 $\pm$ 0.54 \\
							
							&\trtr & {\textbf{50.77 $\pm$ 0.90}}& {\textbf{67.43 $\pm$ 0.89}}	& {\textbf{21.17 $\pm$ 0.38}} &	{\textbf{34.30 $\pm$ 0.41}}	\\
							\bottomrule
							\multirow{3}{*}{\vspace{-0.7em}\rotatebox{90}{tieredImage}}  & method & 1-shot 5-way & 5-shot 5-way& 1-shot 10-way & 5-shot 10-way  \\
							\cline{2-6}

							&	\trval &50.61 $\pm$ 1.12& 67.30	 $\pm$ 0.98&29.18 $\pm$ 0.57& 43.15 $\pm$ 0.72 \\
							
							&\trtr & {\textbf{54.37 $\pm$ 0.93}} & {\textbf{71.45 $\pm$ 0.94}} & {\textbf{35.56 $\pm$ 0.60}} & {\textbf{54.50 $\pm$ 0.71}}  \\
							\bottomrule 
						\end{tabularx} 
				}
                            \end{center}
          \end{threeparttable}                
	\end{table*}
        
\section{Experiments}
\label{appendix:experiments}

\subsection{Simulations}\label{sec:synthetic}
We experiment on the realizable linear model studied in Section \ref{section:realizable}. Recall that the observed data of the $t$-th task are generated as
\begin{align*}
\yb_t = \Xb_t \wb_t, \quad \textrm{with} \quad \xb_{t, i} \overset{\rm iid}{\sim} \normal(0, \Ib_d).
\end{align*}
We independently generate $\wb_{t} \overset{\rm iid}{\sim} \normal(\wb_{0, \star}, \Ib_d / \sqrt{d})$, where $\wb_{0, \star}$ is the linear centroid and the corresponding $R^2 = 1$ here. The goal is to learn the linear centroid $\wb_{0, \star}$ using the \trtr method and \trval method, i.e., minimizing $\hat{L}_{T}^{\nosp}$ and $\hat{L}_{T}^{\sp}$, respectively. Recall that the optimal closed-form solutions $\hat{\wb}_{0, T}^{\{\nosp, \sp\}}$ are given in Section \ref{section:overview-techniques}. We measure the performance of the \trtr and \trval methods using the $\ell_2$-error $\|\wb_{0, \star} - \hat{\wb}_{0, T}^{\{\nosp, \sp\}}\|_2^2$.

\paragraph{Result}
Figure~\ref{fig:simulationappendix} shows the performance of the \trtr and \trval methods on simulated linear centroid meta-learning problems.
Across all simulations, we optimally tune the regularization coefficient $\lambda$ in the \trtr method, and use a sufficiently large $\lambda = 2000$ in the \trval method (according to Lemma~\ref{lemma:optimally-tuned-rates}).
Observe that the MSEs of the two methods decay at rate $O(1/T)$ (Figure~\ref{fig:simulationappendix}(b)). Further, the performance of the two methods in our simulation closely matches the theoretical result in Theorem~\ref{theorem:comparison-mse}, and the \trtr method reliably outperforms the \trval method at all $d/n$ with a moderately large $T$ (Figure~\ref{fig:simulationappendix}(c)).

In Appendix~\ref{appendix:cross-validation}, we additionally investigate the effect of averaging the loss over multiple splits in the \trval method (a ``cross-validation'' type loss). %

\subsection{Few-shot image classification}
\label{section:deep}
We further compare \trtr and \trval type methods on the benchmark few-shot image classification tasks miniImageNet~\citep{ravi2016optimization} and tieredImageNet~\citep{ren2018meta}.

\paragraph{Methods}
We instantiate the \trtr and \trval method in the centroid meta-learning setting with a ridge solver. The methods are almost exactly the same as in our theoretical setting in~\eqref{equation:split-loss} and~\eqref{equation:non-split-loss}, with the only differences being that the parameters $\wb_t$ (and hence $\wb_0$) parametrize a deep neural network instead of a linear classifier, and the loss function is the cross-entropy instead of squared loss. Mathematically, we minimize the following two loss functions:
\begin{align*}
  &  L^{\sp}_{\lambda, n_1}(\wb_0) \defeq \frac{1}{T}\sum_{t=1}\ell_t^{\sp}(\wb_0) = \frac{1}{T}\sum_{t=1}^T \ell\Big( \argmin_{\wb_t} \ell(\wb_t; \Xb_t^{\train}, \yb_t^{\train}) + \lambda\norm{\wb_t - \wb_0}^2 ; \Xb_t^{\val}, \yb_t^{\val}\Big), \\
  &  L^{\nosp}_\lambda(\wb_0) \defeq \frac{1}{T}\sum_{t=1}^T \ell_t^{\nosp}(\wb_0) = \frac{1}{T}\sum_{t=1}^T \ell \Big( \argmin_{\wb_t} \ell(\wb_t; \Xb_t, \yb_t) + \lambda\norm{\wb_t - \wb_0}^2 ; \Xb_t, \yb_t \Big),
\end{align*}
where $(\Xb_t, \yb_t)$ is the data for task $t$ of size $n$, and $(\Xb_t^{\train}, \yb_t^{\train})$ and $(\Xb_t^{\val}, \yb_t^{\val})$ is a split of the data of size $(n_1, n_2)$. We note that both loss functions above have been considered in prior work ($L^{\sp}$ in iMAML~\citep{rajeswaran2019meta}, and $L^{\nosp}$ in Meta-MinibatchProx~\citep{zhou2019efficient}), though we use slightly different implementation details from these prior work to make sure that the two methods here are exactly the same except for whether the split is used. Additional details about the implementation can be found in Appendix~\ref{appendix:deep}.

\paragraph{Experimental settings}
We experiment on miniImageNet~\citep{ravi2016optimization} and tieredImageNet~\citep{ren2018meta} datasets. MiniImageNet consists of $100$ classes of images from ImageNet~\citep{krizhevsky2012imagenet} and each class has $600$ images of resolution $84\times84\times 3$. We use $64$ classes for training, $16$ classes for validation, and the remaining $20$ classes for testing~\citep{ravi2016optimization}. TieredImageNet consists of $608$ classes from the ILSVRC-$12$ data set~\citep{russakovsky2015imagenet} and each image is also of resolution $84\times 84\times 3$.

We adopt the episodic training procedure \citep{finn2017model,zhou2019efficient,rajeswaran2019meta}. In meta-test, we sample a set of $N$-way $(K+1)$-shot test tasks. The first $K$ instances are for training and the remaining one is for testing.  In meta-training, we use the  ``higher way" training strategy. We set the default choice of the train-validation split ratio to be an even split $n_1=n_2=n/2$ following \cite{zhou2019efficient,rajeswaran2019meta}. For example, for a $5$-way $5$-shot classification setting, each task contains $5\times (5+1)=30$ total images, and we set $n_1=n_2=15$. (We additionally investigate the optimality of this split ratio in Appendix~\ref{appendix:split-ratio}.) We report the average accuracy over $2,000$ random test episodes with $95\%$ confidence interval.

\paragraph{Results}
Table~\ref{comparisontable} presents the percent classification accuracy on miniImagenet and tieredImageNet. We find that the \trtr method consistently outperforms the \trval method. Specifically, on miniImageNet, \trtr method outperforms \trval by $2.01\%$ and $3.87\%$ on the $1$-shot $5$-way and $5$-shot $5$-way tasks respectively; On tieredImageNet, \trtr on average improves by about $6.40\%$ on the four testing cases. These results show the advantages of \trtr method over \trval and support our theoretical findings in Theorem~\ref{theorem:comparison-mse}.

\ifdefined\minusspace\vspace{-0.7em}\else\fi
\section{Conclusion}
\ifdefined\minusspace\vspace{-0.3em}\else\fi
We study the importance of train-validation split on the linear-centroid meta-learning problem, and show that the necessity and optimality of train-validation split depends greatly on whether the tasks are structured: the sample splitting is necessary in general situations, and not necessary and non-optimal when the tasks are nicely structured. It would be of interest to study whether similar conclusions hold on other meta-learning problems such as learning representations, or how our insights can guide the design of meta-learning algorithms with better empirical performance. 

\bibliography{bib}
\bibliographystyle{plainnat}
\appendix

\makeatletter
\def\renewtheorem#1{%
  \expandafter\let\csname#1\endcsname\relax
  \expandafter\let\csname c@#1\endcsname\relax
  \gdef\renewtheorem@envname{#1}
  \renewtheorem@secpar
}
\def\renewtheorem@secpar{\@ifnextchar[{\renewtheorem@numberedlike}{\renewtheorem@nonumberedlike}}
\def\renewtheorem@numberedlike[#1]#2{\newtheorem{\renewtheorem@envname}[#1]{#2}}
\def\renewtheorem@nonumberedlike#1{  
\def\renewtheorem@caption{#1}
\edef\renewtheorem@nowithin{\noexpand\newtheorem{\renewtheorem@envname}{\renewtheorem@caption}}
\renewtheorem@thirdpar
}
\def\renewtheorem@thirdpar{\@ifnextchar[{\renewtheorem@within}{\renewtheorem@nowithin}}
\def\renewtheorem@within[#1]{\renewtheorem@nowithin[#1]}
\makeatother

\renewtheorem{theorem}{Theorem}[section]
\renewtheorem{lemma}{Lemma}[section]
\renewtheorem{remark}{Remark}
\renewtheorem{corollary}{Corollary}[section]
\renewtheorem{observation}{Observation}[section]
\renewtheorem{proposition}{Proposition}[section]
\renewtheorem{definition}{Definition}[section]
\renewtheorem{claim}{Claim}[section]
\renewtheorem{fact}{Fact}[section]
\renewtheorem{assumption}{Assumption}[section]
\renewcommand{\theassumption}{\Alph{assumption}}
\renewtheorem{conjecture}{Conjecture}[section]

\section{Proof of Theorem~\ref{theorem:agnostic}}
\label{appendix:proof-agnostic}

\subsection{Proof of part (a)}

We need to show that
\begin{align*}
  L^{\sp}(\wb_0) = \E[\ell_t^{\sp}(\wb_0)] = L^{\test}_{\lambda, n_1}(\wb_0)
\end{align*}
for all $\wb_0$, that is, the population meta-test loss is exactly the
same as the expected loss of the \trval method. This is straightforward: as the tasks are i.i.d. and $\cA_\lambda(\wb_0; \Xb_t^{\train}, \yb_t^{\train})$ is independent of the test points $(\Xb_t^{\val}, \yb_t^{\val})$, we have for any $\wb_0$ that
\begin{align*}
  & \quad \E[\ell_t^{\sp}(\wb_0)] = \E_{p_t\sim\Pi, (\Xb_t, \yb_t)\sim p_t}\brac{ \frac{1}{2n_2}\norm{\yb_t^{\val} - \Xb_t^{\val} \cA_\lambda(\wb_0; \Xb_t^{\train}, \yb_t^{\train})}^2 } \\
  & = \E_{p_t\sim\Pi, (\Xb_t, \yb_t)\sim p_t} \brac{ \frac{1}{2}\paren{ y^{\val}_{t,1} - \xb^{\val\top}_{t,1}\cA_\lambda(\wb_0; \Xb_t^{\train}, \yb_t^{\train}) }^2 } \\
  & = \E_{p_{T+1}\sim\Pi, (\Xb_{T+1}, \yb_{T+1}), (\xb',y')\simiid p_t} \brac{ \frac{1}{2}\paren{ y' - \xb'^\top \cA_{\lambda, n_1}(\wb_0; \Xb_{T+1}, \yb_{T+1})}^2 } \\
  & = L^{\test}_{\lambda, n_1}(\wb_0).
\end{align*}
This finishes the proof of part (a). \qed

We also calculate the minimizer of the test-time loss $L^{\test}_{\lambda, n}$ (notice that here we use $n$ instead of $n_1$ training samples per task), which will be useful for our proof of part (b). We have
\begin{align*}
  & \quad L^{\test}_{\lambda, n}(\wb_0) = \E_{p_t\sim\Pi, (\Xb_t, \yb_t), (\xb',y')\simiid p_t} \brac{ \frac{1}{2}\paren{ y' - \xb'^\top \cA_{\lambda, n}(\wb_0; \Xb_{t}, \yb_{t})}^2 } \\
  & = \E_{p_t\sim\Pi, (\Xb_t, \yb_t), (\xb',y')\simiid p_t}\brac{  \frac{1}{2} \paren{ y' - \xb'^\top \brac{ \wb_0 + (\Xb_{t}^\top\Xb_{t} + n\lambda\Ib_d)^{-1}\Xb_t^\top(\yb_t - \Xb_t\wb_0) } }^2 } \\
  & = \frac{1}{2}\wb_0^\top \Mb \wb_0 - \wb_0^\top \bbb + {\rm const},
\end{align*}
where
\begin{equation}
  \label{equation:m-expectation}
  \begin{aligned}
    & \quad \Mb \defeq \E_{p_t\sim\Pi, (\Xb_t, \yb_t), (\xb',y')\simiid p_t}\brac{ \paren{ \Ib_d - (\Xb_{t}^\top\Xb_{t} + n\lambda\Ib_d)^{-1}\Xb_t^\top\Xb_t}\xb'\xb'^\top \paren{ \Ib_d - (\Xb_{t}^\top\Xb_{t} + n\lambda\Ib_d)^{-1}\Xb_t^\top\Xb_t} } \\
    & = \E_{p_t, (\Xb_t, \yb_t)} \brac{\lambda^2 (\Xb_t^\top\Xb_t/n + \lambda \Ib_d)^{-1} \bSigma_t(\Xb_t^\top\Xb_t/n + \lambda \Ib_d)^{-1} } \succ \bzero,
  \end{aligned}
\end{equation}
where $\bSigma_t\defeq \E_{x\sim p_t}[\xb\xb^\top] \succ \bzero$, and
\begin{equation}
  \label{equation:b-expectation}
  \begin{aligned}
    & \quad \bbb \defeq \E_{p_t\sim\Pi, (\Xb_t, \yb_t), (\xb',y')\simiid p_t}\brac{ \paren{ \Ib_d - (\Xb_{t}^\top\Xb_{t} + n\lambda\Ib_d)^{-1} } \xb_t'\paren{ y_t' - \xb_t'^\top (\Xb_{t}^\top\Xb_{t} + n\lambda\Ib_d)^{-1} \Xb_t^\top\yb_t } } \\
    & = \E_{p_t\sim\Pi, (\Xb_t, \yb_t), (\xb',y')\simiid p_t}\brac{ \lambda \paren{ \Xb_t^\top\Xb_t/n + \lambda\Ib_d}^{-1} \xb_t'\paren{ y_t' - \xb_t'^\top (\Xb_{t}^\top\Xb_{t}/n + \lambda\Ib_d)^{-1} \Xb_t^\top\yb_t/n } } \\
    & = \lambda\E\brac{(\Xb_t^\top\Xb_t/n + \lambda\Ib_d)^{-1}}\E_{p_t,(\xb',y')\sim p_t}[\xb' y'] - \lambda\E\brac{ (\Xb_t^\top\Xb_t/n + \lambda\Ib_d)^{-1}\bSigma_t (\Xb_t^\top\Xb_t/n + \lambda\Ib_d)^{-1}\frac{1}{n}\Xb_t^\top\yb_t }.
  \end{aligned}
\end{equation}
Noticing that $(\Xb_t^{\train^\top}\Xb_t^{\train}/n+\lambda\Ib_d)^{-1}\preceq \lambda^{-1}\Ib_d$ and by the assumptions that $\E_{(\xb,y)\sim p_t}[\opnorm{\xb\xb^\top}]< \infty$, $\E_{(\xb,y)\sim p_t}[\norm{\xb y}] < \infty$, we have $\opnorm{\Mb}< \infty$ and $\norm{\bbb_T}<\infty$. Therefore, the minimizer of $L^{\test}_{\lambda, n}$ is
\begin{align}
  \label{equation:test-minimizer}
  \wb_{0,\star}(\lambda, n) = \argmin_{\wb_0} L^{\test}_{\lambda, n}(\wb_0) = \Mb^{-1}\bbb,
\end{align}
where $\Mb\in\R^{d\times d}$ and $\bbb\in\R^d$ are defined in~\eqref{equation:m-expectation} and~\eqref{equation:b-expectation}.

\subsection{Proof of part (b)}

We construct a simple counter-example on which the minimizer of $L^{\nosp}$ is not equal to that of $L^{\test}_{\lambda, n}$ for any $\lambda>0$ and $n\ge 1$. We begin by simplifying the \trtr loss. We have
\begin{align*}
  & \quad \ell_t^{\nosp}(\wb_0) = \frac{1}{2n}\norm{ \yb_t - \Xb_t\cA_\lambda(\wb_0; \Xb_t, \yb_t) }^2 \\
  & = \frac{1}{2n}\norm{\yb_t - \Xb_t\brac{\wb_0 + (\Xb_t^\top\Xb_t + n\lambda\Ib_d)^{-1}\Xb_t^\top(\yb_t - \Xb_t\wb_0)}}^2 \\
  & = \frac{1}{2}\norm{\Ab_t\wb_0 - \cbb_t}^2,
\end{align*}
where
\begin{align*}
  \Ab_t = \frac{1}{\sqrt{n}} n\lambda\Xb_t(\Xb_t^\top\Xb_t + n\lambda\Ib_d)^{-1}~~~{\rm and}~~~\cbb_t = \frac{1}{\sqrt{n}}\paren{\Ib_n - \Xb_t(\Xb_t^\top\Xb_t + n\lambda\Ib_d)^{-1}\Xb_t^\top}\yb_t.
\end{align*}
Therefore, the minimizer of the expected loss $L^{\nosp}$ is
\begin{equation}
  \label{equation:l1-global-min}
\begin{aligned}
  & \wb_{0,\star}^{\nosp} = \argmin_{\wb_0} L^{\nosp}(\wb_0) = \paren{\E[\Ab_t^\top\Ab_t]}^{-1}\E[\Ab_t^\top\cbb_t] \\
  & = \E\brac{ \lambda^2(\Xb_t^\top\Xb_t/n + \lambda\Ib_d)^{-2}\frac{\Xb_t^\top\Xb_t}{n} }^{-1} \cdot \E\brac{ \frac{1}{n}\lambda(\Xb_t^\top\Xb_t/n+\lambda\Ib_d)^{-1}\Xb_t^\top(\Ib_n - \Xb_t(\Xb_t^\top\Xb_t+n\lambda\Ib_d)^{-1}\Xb_t^\top)\yb_t } \\
  & = \E\brac{ \lambda^2(\Xb_t^\top\Xb_t/n + \lambda\Ib_d)^{-2}\frac{\Xb_t^\top\Xb_t}{n} }^{-1} \cdot \E\brac{ \lambda^2 (\Xb_t^\top\Xb_t/n + \lambda\Ib_d)^{-2}\frac{1}{n}\Xb_t^\top\yb_t }.
\end{aligned}
\end{equation}
On the other hand, recall from~\eqref{equation:test-minimizer}
that the minimizer of the test-time loss $L^{\test}_{\lambda, n}$ is
\begin{equation}
  \label{equation:l2-global-min}
\begin{aligned}
  & \quad \wb_{0,\star}(\lambda, n) = \argmin_{\wb_0} L^{\test}_{\lambda, n}(\wb_0) \\
  & = \E\brac{ \lambda^2(\Xb_t^\top\Xb_t/n + \lambda\Ib_d)^{-1}\bSigma_t(\Xb_t^\top\Xb_t/n + \lambda\Ib_d)^{-1}}^{-1} \cdot \bigg\{ \lambda\E\brac{(\Xb_t^\top\Xb_t/n + \lambda\Ib_d)^{-1}}\E_{p_t,(\xb',y')\sim p_t}[\xb' y'] \\
  & \qquad - \lambda\E\brac{ (\Xb_t^\top\Xb_t/n + \lambda\Ib_d)^{-1}\bSigma_t (\Xb_t^\top\Xb_t/n + \lambda\Ib_d)^{-1}\frac{1}{n}\Xb_t^\top\yb_t }\bigg\}.
\end{aligned}
\end{equation}

\paragraph{Construction of the counter-example}
We now construct a distribution for which~\eqref{equation:l1-global-min} is not equal to~\eqref{equation:l2-global-min}. Let $d=1$ and let all $p_t$ be the following distribution:
\begin{align*}
  p_t:~~~(x_{t,i},y_{t,i}) = \left\{
  \begin{aligned}
    & (1, 3) ~~~\textrm{with probability}~1/2; \\
    & (3, -1) ~~~\textrm{with probability}~1/2.
  \end{aligned}
  \right.
\end{align*}
Clearly, we have $\bSigma_t=5$, $s_t\defeq\Xb_t^\top\Xb_t/n \in [1, 9]$, and $\E_{x',y'\sim p_t}[x'y']=0$. Therefore we have
\begin{align*}
  \wb_{0,\star}^{\nosp} = \E\brac{ (s_t+\lambda)^{-2}s_t }^{-1} \cdot \E\brac{(s_t+\lambda)^{-2} \frac{1}{n}\sum_{i=1}^n x_{t,i}y_{t,i}},
\end{align*}
and
\begin{align*}
  & \quad \wb_{0,\star}(\lambda, n) = -\E\brac{ 5\lambda^2(s_t + \lambda)^{-2} }^{-1} \cdot \E\brac{ 5\lambda(s_t+\lambda)^{-2}\frac{1}{n}\sum_{i=1}^n x_{t,i}y_{t,i} } \\
  & = -\E\brac{\lambda(s_t + \lambda)^{-2}}^{-1}\cdot \E\brac{ (s_t + \lambda)^{-2} \frac{1}{n}\sum_{i=1}^n x_{t,i}y_{t,i} }.
\end{align*}
We now show that $\wb_{0,\star}^{\nosp}\neq \wb_{0,\star}(\lambda, n)$ by showing that
\begin{align*}
  \E\brac{ (s_t + \lambda)^{-2} \frac{1}{n}\sum_{i=1}^n x_{t,i}y_{t,i} } = \E\brac{ \frac{x_{t,1}y_{t,1}}{(s_t+\lambda)^2} } \neq 0
\end{align*}
for any $\lambda>0$. Indeed, conditioning on $(x_{t,1},y_{t,1})=(1,3)$, we know that the sum-of-squares in $s_t$ has one term that equals $1$, and all others i.i.d. being $1$ or $9$ with probability one half. On the other hand, if we condition on $(x_{t,1}, y_{t,1})=(3,-1)$, then we know the sum in $s_t$ has one term that equals $9$ and all others i.i.d.. This means that the negative contribution in the expectation is smaller than the positive contribution, in other words
\begin{align*}
  & \E\brac{ \frac{x_{t,1}y_{t,1}}{(s_t+\lambda)^2}} = \frac{1}{2} \cdot 3\E\brac{ \frac{1}{(s_t+\lambda)^2} \bigg| (x_{t,1}, y_{t,1})=(1,3)}  \\
  & \qquad + \frac{1}{2} \cdot -3\E\brac{ \frac{1}{(s_t+\lambda)^2} \bigg| (x_{t,1}, y_{t,1})=(3,-1)} > 0.
\end{align*}
This shows $\wb_{0,\star}^{\nosp}\neq \wb_{0,\star}(\lambda, n)$.

Finally, for this distribution, the test loss $L^{\test}_{\lambda, n}(\wb_0)$ is strongly convex (since it has a positive second derivative), this further implies that the excess loss $L^{\test}_{\lambda, n}(\wb_{0,\star}^{\nosp}) - L^{\test}_{\lambda, n}(\wb_{0,\star}(\lambda, n))$ is bounded away from zero.
\qed
\section{Proof of Theorem~\ref{theorem:consistency-realizable-model}}
\label{appendix:proof-consistency-realizable-model}

We first show that $\wb_{0, \star} = \EE_{\wb_t \sim \Pi}[\wb_t]$ is a global optimizer for $L^{\nosp}_{\lambda, n}$ and $L^{\sp}_{\lambda, n_1, n_2}$ with any regularization coefficient $\lambda>0$, any $n$, and any split $(n_1, n_2)$. To do this, it suffices to check that the gradient at $\wb_{0, \star}$ is zero and the Hessian is positive definite (PD).

\noindent {\bf Optimality of $\wb_{0, \star}$ in both $L^{\nosp}_{\lambda, n}$ and $L^{\sp}_{\lambda, n_1, n_2}$}.
We first look at $L^{\nosp}$: for any $\wb_0\in\R^d$ we have
\begin{align}
L^{\nosp}_{\lambda, n}(\wb_0) & = \E[\ell^{\nosp}_{t}(\wb_0)] \notag \\
& = \frac{1}{2n} \E\left[\Big\|\Xb_t \wb_t - \Xb_t\left[\left(\Xb_t^\top \Xb_t + n\lambda\Ib_d \right)^{-1} \Xb_t^\top \left(\Xb_t \wb_t - \Xb_t \wb_0\right) + \wb_0\right]\Big\|_2^2\right] \notag \\
& = \frac{1}{2n} \E\Big[\Big\|\Xb_t\left(\Ib_d - \left(\Xb_t^\top \Xb_t + n\lambda \Ib_d\right)^{-1} \Xb_t^\top \Xb_t \right)(\wb_t - \wb_0)\Big\|_2^2\Big]. \label{eq:L1rewrite}
\end{align}
Similarly, $L^{\sp}_{\lambda, n_1, n_2}$ can be written as
\begin{align}
L^{\sp}(\wb_0) & = \E[\ell^{\sp}_{t}(\wb_0)] \notag \\
& = \frac{1}{2n_2} \E\left[\Big\|\Xb^{\val}_t \wb_t - \Xb^{\val}_t\left[\left((\Xb_t^{\train})^\top \Xb^{\train}_t + n_1 \lambda\Ib_d \right)^{-1} (\Xb_t^{\train})^\top \left(\Xb^{\train}_t \wb_t - \Xb^{\train}_t \wb_0\right) + \wb_0\right]\Big\|_2^2\right] \notag \\
& = \frac{1}{2n_2} \E\Big[\Big\|\Xb^{\val}_t\left(\Ib_d - \left((\Xb_t^{\train})^\top \Xb^{\train}_t + n_1 \lambda \Ib_d\right)^{-1} (\Xb_t^{\train})^\top \Xb^{\train}_t \right)(\wb_t - \wb_0)\Big\|_2^2 \Big]. \label{eq:L2rewrite}
\end{align}
We denote 
\begin{align*}
\Mb_t^{\nosp} & = \Xb_t\left(\Ib_d - \left(\Xb_t^\top \Xb_t + n\lambda \Ib_d\right)^{-1} \Xb_t^\top \Xb_t \right) \quad \textrm{and} \\
\Mb_t^{\sp} & = \Xb^{\val}_t\left(\Ib_d - \left((\Xb_t^{\train})^\top \Xb^{\train}_t + n_1 \lambda \Ib_d\right)^{-1} (\Xb_t^{\train})^\top \Xb^{\train}_t \right)
\end{align*}
to simplify the notations in \eqref{eq:L1rewrite} and \eqref{eq:L2rewrite}. We take gradient of $L^{\nosp}$ and $L^{\sp}$ with respect to $\wb_0$:
\begin{align}
& \nabla_{\wb_0} L^{\nosp}_{\lambda, n}(\wb_0) = -\frac{1}{n} \E \left[(\Mb_t^{\nosp})^\top \Mb_t^{\nosp} (\wb_t - \wb_0) \right], \label{eq:gradL1} \\
& \nabla_{\wb_0} L^{\sp}_{\lambda, n_1, n_2}(\wb_0) = -\frac{1}{n_2} \E \left[(\Mb_t^{\sp})^\top \Mb_t^{\sp}(\wb_t - \wb_0) \right] \label{eq:gradL2}.
\end{align}
Substituting $\wb_{0, \star}$ into \eqref{eq:gradL1} and taking expectation, we deduce
\begin{align}\label{eq:gradL1vanish}
\nabla_{\wb_0} L^{\nosp}_{\lambda, n}(\wb_{0, \star}) = -\frac{1}{n} \E\left[(\Mb_t^{\nosp})^\top \Mb_t^{\nosp} (\wb_t - \wb_{0, \star})\right] = \boldsymbol{0}.
\end{align}
To see this, observe that by definition $\E[\wb_t - \wb_{0, \star}] = \boldsymbol{0}$. Combining with $\wb_t$ being generated independently of $\Xb_t$, we obtain that the RHS of \eqref{eq:gradL1vanish} vanish. Following the same argument, we can show
\begin{align}
\nabla_{\wb_0} L^{\sp}_{\lambda, n_1, n_2}(\wb_{0, \star}) = \boldsymbol{0}, \nonumber
\end{align}
since $\Xb^{\val}_t$ is also independent of $\wb_t$. The reasoning above indicates that $\wb_{0, \star}$ is a stationary point of both $L^{\nosp}_{\lambda, n}$ and $L^{\sp}_{\lambda, n_1, n_2}$. The remaining step is to check $\nabla_{\wb_0}^2 L^{\nosp}_{\lambda, n}(\wb_{0, \star})$ and $\nabla_{\wb_0}^2 L^{\sp}_{\lambda, n_1, n_2}(\wb_{0, \star})$ are positive definite. From \eqref{eq:gradL1} and \eqref{eq:gradL2}, we derive respectively the Hessian of $L^{\nosp}_{\lambda, n}$ and $L^{\sp}_{\lambda, n_1, n_2}$ as
\begin{align}
\nabla^2_{\wb_0} L^{\nosp}_{\lambda, n}(\wb_{0, \star}) & = \frac{1}{n} \E[(\Mb_t^{\nosp})^\top \Mb_t^{\nosp}] \quad \textrm{and} \notag \\
\nabla^2_{\wb_0} L^{\sp}_{\lambda, n_1, n_2}(\wb_{0, \star}) & = \frac{1}{n_2} \E[(\Mb_t^{\sp})^\top \Mb_t^{\sp}]. \notag
\end{align}
Let $\vb \in \RR^d$ be any nonzero vector, our goal is to check $\vb^\top \nabla^2_{\wb_0} L^{\nosp}_{\lambda, n}(\wb_{0, \star}) \vb > 0$. A key observation is that $\left(\Ib_d - \left(\Xb_t^\top \Xb_t + n\lambda \Ib_d\right)^{-1} \Xb_t^\top \Xb_t \right)$ is positive definite for any $\lambda \neq 0$. To see this, let $\sigma_1 \geq \dots \geq \sigma_d$ be eigenvalues of $\frac{1}{n} \Xb_t^\top \Xb_t$, some algebra yields the eigenvalues of $\left(\Ib_d - \left(\Xb_t^\top \Xb_t + n\lambda \Ib_d\right)^{-1} \Xb_t^\top \Xb_t \right)$ are $\frac{\lambda}{\lambda + \sigma_i} > 0$ for $\lambda \neq 0$ and $i = 1, \dots, d$. Hence, we deduce
\begin{align}\label{eq:pdhessian}
\vb^\top \nabla^2_{\wb_0} L^{\nosp}(\wb_{0, \star}) \vb = \frac{1}{n} \EE[\vb^\top \Xb_t^\top \left(\Ib_d - \left(\Xb_t^\top \Xb_t + n\lambda \Ib_d\right)^{-1} \Xb_t^\top \Xb_t \right)^2 \Xb_t \vb] > 0.
\end{align}
The detailed computation of the eigenvalues of $\left(\Ib_d - \left(\Xb_t^\top \Xb_t + n\lambda \Ib_d\right)^{-1} \Xb_t^\top \Xb_t \right)$ utilizes the assumption that $\Xb_t$ is isotropic (an explicit argument is deferred to the proof of Lemma \ref{lemma:moments}). As a consequence, we have shown that $\wb_{0, \star}$ is a global optimum of $L^{\nosp}_{\lambda, n}$. The same argument applies to $L^{\sp}_{\lambda, n_1, n_2}$, and the proof is complete.

\section{Proof of Theorem~\ref{theorem:concentration-mse}}
\label{appendix:proof-concentration-mse}

The proof is organized as follows. We first derive the closed-form expressions of the \trtr and \trval estimators in terms of the matrices $\Ab_t$, $\Bb_t$ in Section~\ref{appendix:closed-form-estimators}. We compute the first and second moments of $\Ab_t$ and $\Bb_t$ in Section~\ref{appendix:moments}, present some concentration lemmas in Section~\ref{appendix:concentration-lemmas}, and then prove the main theorem in Section~\ref{appendix:proof-main-concentration-mse}.

\subsection{Closed-form expressions for the estimators}
\label{appendix:closed-form-estimators}

\begin{lemma}[Closed-form expressions for $\hat{\wb}_{0,T}^{\nosp}$ and $\hat{\wb}_{0,T}^{\sp}$]
  \label{lemma:closed-form-estimators}
  For the realizable linear model~\eqref{equation:realizable-model}, the \trtr method~\eqref{equation:non-split-loss} and the \trval method~\eqref{equation:split-loss} have closed-form expressions
  \begin{align}
    & \hat{\wb}^{\nosp}_{0,T} = \paren{\sum_{t=1}^T \Ab_t}^{-1} \sum_{t=1}^T \Ab_t\wb_t,~\label{equation:closed-form-trtr} \\
    & \hat{\wb}^{\sp}_{0,T} = \paren{\sum_{t=1}^T \Bb_t}^{-1} \sum_{t=1}^T \Bb_t \wb_t,~\label{equation:closed-form-trval}
\end{align}
where
\begin{align}
  & \Ab_t \defeq \lambda^2 \paren{ \frac{\Xb_t^\top\Xb_t}{n} + \lambda\Ib_d}^{-2} \frac{\Xb_t^\top\Xb_t}{n},~\label{equation:at} \\
  & \Bb_t \defeq \lambda^2 \paren{ \frac{\Xb_t^{\train\top}\Xb_t^{\train}}{n_1} + \lambda\Ib_d}^{-1} \frac{\Xb_t^{\val\top}\Xb_t^{\val}}{n_2} \paren{ \frac{\Xb_t^{\train\top}\Xb_t^{\train}}{n_1} + \lambda\Ib_d}^{-1}.~\label{equation:bt}
\end{align}
\end{lemma}
\begin{proof}
We consider the \trtr method first. Substituting $\cA_{\lambda}(\wb_0; \Xb, \yb) = \wb_0 + (\Xb^\top\Xb + n\lambda \Ib_d)^{-1}\Xb^\top(\yb - \Xb\wb_0)$ into \eqref{equation:non-split-loss} yields
\begin{align}
\hat{\wb}^{\nosp}_{0, T} = \argmin_{\wb_0} \frac{1}{T} \sum_{t=1}^T \frac{1}{2n} \norm{\yb_t - \Xb_t \left(\wb_0 + (\Xb_t^\top\Xb_t + n\lambda \Ib_d)^{-1}\Xb_t^\top(\yb_t - \Xb_t\wb_0)\right)}^2. \nonumber
\end{align}
The optimization problem above is quadratic in $\wb_0$. Therefore, by setting the gradient with respect to $\wb_0$ equal to zero, we derive
\begin{align}
\hat{\wb}_{0, T}^{\nosp} & = \argmin_{\wb_0} \frac{1}{T} \sum_{t=1}^T \frac{1}{2n} \norm{\left(\Ib_d - \Xb_t (\Xb_t^\top \Xb_t + n\lambda \Ib_d)^{-1} \Xb_t^\top\right) \yb_t - \Xb_t \left(\Ib_d - (\Xb_t^\top \Xb_t + n\lambda \Ib_d)^{-1} \Xb_t^\top \Xb_t \right) \wb_0}^2 \nonumber \\
& \overset{(i)}{=} \argmin_{\wb_0} \frac{1}{T} \sum_{t=1}^T \frac{1}{2n} \norm{\Xb_t \left(\Ib_d - (\Xb_t^\top \Xb_t + n\lambda \Ib_d)^{-1} \Xb_t^\top \Xb_t \right) \wb_t - \Xb_t \left(\Ib_d - (\Xb_t^\top \Xb_t + n\lambda \Ib_d)^{-1} \Xb_t^\top \Xb_t \right) \wb_0}^2 \nonumber \\
& \overset{(ii)}{=} \argmin_{\wb_0} \frac{1}{T} \sum_{t=1}^T \frac{1}{2n} \norm{\lambda \Xb_t \left(\Xb_t^\top \Xb_t/n + \lambda \Ib_d \right)^{-1} (\wb_t - \wb_0)}^2 \nonumber \\
& = \left(\sum_{t=1}^T \lambda^2 \left(\frac{\Xb_t^\top \Xb_t}{n} + \lambda \Ib_d \right)^{-1} \frac{\Xb_t^\top \Xb_t}{n} \left(\frac{\Xb_t^\top \Xb_t}{n} + \lambda \Ib_d \right)^{-1}\right)^{-1} \nonumber \\
& \qquad \cdot \sum_{t=1}^T \lambda^2 \left(\frac{\Xb_t^\top \Xb_t}{n} + \lambda \Ib_d \right)^{-1} \frac{\Xb_t^\top \Xb_t}{n} \left(\frac{\Xb_t^\top \Xb_t}{n} + \lambda \Ib_d \right)^{-1} \wb_t \nonumber \\
& \overset{(iii)}{=} \left(\sum_{t=1}^T \Ab_t\right)^{-1} \sum_{t=1}^T \Ab_t \wb_t, \nonumber
\end{align}
where step $(i)$ invokes the data generating assumption $\yb_t = \Wb_t \wb_t$, step $(ii)$ simplifies $\Ib_d - (\Xb_t^\top \Xb_t + n\lambda \Ib_d)^{-1} \Xb_t^\top \Xb_t$ as $\lambda n (\Xb_t^\top \Xb_t + n\lambda \Ib_d)$ by writing $\Ib_d = (\Xb_t^\top \Xb_t + n\lambda \Ib_d)^{-1} (\Xb_t^\top \Xb_t + n\lambda \Ib_d)$, and step $(iii)$ follows from plugging in the definition of $\Ab_t$ in \eqref{equation:at} and the fact that $(\Xb_t^\top \Xb_t + n\lambda \Ib_d)^{-1}$ and $\Xb_t^\top \Xb_t$ commute.

Next we consider the \trval method. The argument is analogous to the \trtr method. In particular, we recall from \eqref{equation:split-loss}:
\begin{align}
\hat{\wb}_{0, T}^{\sp} = \argmin_{\wb_0} \frac{1}{T} \sum_{t=1}^T \frac{1}{2n_2} \norm{\yb_t^{\val} - \Xb^{\val}_t \left(\wb_0 + ((\Xb^{\train}_t)^\top\Xb^{\train}_t + n_1\lambda \Ib_d)^{-1}(\Xb^{\train}_t)^\top(\yb^{\train}_t - \Xb^{\train}_t \wb_0)\right)}^2. \nonumber
\end{align}
The optimization problem above is still quadratic in $\wb_0$. Using the same rearrangement technique for the \trtr method, we deduce
\begin{align}
\hat{\wb}_{0, T}^{\sp} & = \argmin_{\wb_0} \frac{1}{T} \sum_{t=1}^T \frac{1}{2n_2} \norm{\lambda \Xb^{\val}_t \left((\Xb_t^{\train})^\top \Xb^{\train}_t/n_1 + \lambda \Ib_d \right)^{-1} (\wb_t - \wb_0)}^2 \nonumber \\
& = \left(\sum_{t=1}^T \lambda^2 \left(\frac{(\Xb^{\train}_t)^\top \Xb^{\train}_t}{n_1} + \lambda \Ib_d \right)^{-1} \frac{(\Xb^{\val}_t)^\top \Xb^{\val}_t}{n_2} \left(\frac{(\Xb^{\train}_t)^\top \Xb^{\train}_t}{n_1} + \lambda \Ib_d \right)^{-1}\right)^{-1} \nonumber \\
& \qquad \cdot \sum_{t=1}^T \lambda^2 \left(\frac{(\Xb^{\train}_t)^\top \Xb^{\train}_t}{n_1} + \lambda \Ib_d \right)^{-1} \frac{(\Xb^{\val}_t)^\top \Xb^{\val}_t}{n_2} \left(\frac{(\Xb^{\train}_t)^\top \Xb^{\train}_t}{n_1} + \lambda \Ib_d \right)^{-1} \wb_t \nonumber \\
& = \left(\sum_{t=1}^T \Bb_t\right)^{-1} \sum_{t=1}^T \Bb_t \wb_t, \nonumber
\end{align}
where the last equality follows by substituting the definition of matrix $\Bb_t$ in \eqref{equation:bt}. The proof is complete.
\end{proof}

\subsection{Moments of $\Ab_t$ and $\Bb_t$}
\label{appendix:moments}

\begin{lemma}[Moments of $\Ab_t$ and $\Bb_t$]
  \label{lemma:moments}
  Suppose $\xb_{t,i}\simiid \normal(0,\Ib_d)$ and $\Ab_t$, $\Bb_t$ are defined as in~\eqref{equation:at} and~\eqref{equation:bt}. Then, we have
  \begin{align*}
    & \E\brac{\Ab_t} = f_A(n,d,\lambda) \cdot \Ib_d,~~~\E\brac{\Ab_t^2} = f_{A^2}(n,d,\lambda) \cdot \Ib_d,\\
    & \E\brac{\Bb_t} = f_{B}(n_1, n_2,d,\lambda) \cdot \Ib_d,~~~\E\brac{\Bb_t^2} = f_{B^2}(n_1, n_2, d, \lambda) \cdot \Ib_d,
  \end{align*}
  where
  \begin{align*}      
    & f_A(n,d,\lambda) \defeq \frac{1}{d}\E\brac{\tr\paren{ \lambda^2(\hat{\bSigma}_n + \lambda\Ib_d)^{-2}\hat{\bSigma}_n }} = \frac{1}{d} \E\brac{ \sum_{i=1}^d \sigma_i^{(n)}\lambda^2/(\sigma_i^{(n)} + \lambda)^2 }, \\
    & f_{A^2}(n,d,\lambda) \defeq \frac{1}{d}\E\brac{\tr\paren{ \lambda^4(\hat{\bSigma}_n + \lambda\Ib_d)^{-4}\hat{\bSigma}_n^2 }}  = \frac{1}{d} \E\brac{ \sum_{i=1}^d (\sigma_i^{(n)})^2\lambda^4/(\sigma_i^{(n)} + \lambda)^4 } , \\
    & f_B(n_1, n_2, d, \lambda) \defeq \frac{1}{d} \E\brac{ \tr\paren{ \lambda^2(\hat{\bSigma}_{n_1} + \lambda\Ib_d)^{-2} } } = \frac{1}{d} \E\brac{ \sum_{i=1}^d \lambda^2 /(\sigma_i^{(n_1)} + \lambda)^2 }, \\
    & f_{B^2}(n_1, n_2, d, \lambda) \defeq \frac{1}{dn_2}\E\brac{ \tr\paren{\lambda^2 (\hat{\bSigma}_{n_1} + \lambda\Ib_d)^{-2} }^2 + (n_2+1) \tr\paren{ \lambda^4(\hat{\bSigma}_{n_1} + \lambda\Ib_d)^{-4} } } \\
    & \qquad \qquad \qquad = \frac{1}{dn_2} \E\brac{ \paren{\sum_{i=1}^d \lambda^2 /(\sigma_i^{(n_1)} + \lambda)^2}^2 + (n_2+1)\sum_{i=1}^d \lambda^4/(\sigma_i^{(n_1)} + \lambda)^4 },
  \end{align*}
  where $\hat{\bSigma}_n$ denotes the empirical covariance matrix $\Xb_t^\top\Xb_t/n$ where $\Xb_t\in\R^{n\times d}$ has i.i.d. $\normal(0,1)$ entries, and $\sigma_1^{(n)}\ge \dots\ge\sigma_d^{(n)}\ge 0$ is its eigenvalues. 
\end{lemma}
\begin{proof}
The proof manipulates the isotropicity of $\Xb_t$. We begin with the first moment computation.

$\bullet$ {\bf First moment of $\Ab_t$ and $\Bb_t$}.
We rewrite $\Ab_t$ in a symmetric form to ease the analysis:
\begin{align}
\EE[\Ab_t] \notag & = \EE\left[\left(\Ib_d - \left(\Xb_t^\top \Xb_t + n\lambda \Ib_d\right)^{-1} \Xb_t^\top \Xb_t \right)^\top \frac{\Xb_t^\top \Xb_t}{n} \left(\Ib_d - \left(\Xb_t^\top \Xb_t + n\lambda \Ib_d\right)^{-1} \Xb_t^\top \Xb_t \right)\right] \notag \\
& \overset{(i)}{=} \frac{1}{n} \EE\left[\Vb_t \left(\Ib_d - (\Db_t^\top \Db_t + n\lambda \Ib_d)^{-1} \Db_t^\top \Db_t \right)^\top \Db_t^\top \Db_t \left(\Ib_d - (\Db_t^\top \Db_t + n\lambda \Ib_d)^{-1} \Db_t^\top \Db_t \right) \Vb_t^\top \right], \label{eq:MTMSVD}
\end{align}
where the equality $(i)$ is obtained by plugging in the SVD of $\Xb_t = \Ub_t \Db_t \Vb_t^\top$ with  $\Ub_t \in \RR^{n \times n}$, $\Db_t \in \RR^{n \times d}$, and $\Vb_t \in \RR^{d \times d}$. A key observation is that $\Ub_t$ and $\Vb_t$ are independent, since $\Xb_t$ is isotropic, i.e., homogeneous in each orthogonal direction. To see this, for any orthogonal matrices $\Qb \in \RR^{n \times n}$ and $\Pb \in \RR^{d \times d}$, we know $\Xb_t$ and $\Qb \Xb_t \Pb^\top$ share the same distribution. Moreover, we have $\Qb \Xb_t \Pb^\top= (\Qb \Ub_t) \Db_t (\Pb \Vb_t)^\top$ as the SVD. This shows that the left and right singular matrices are independent and both uniformly distributed on all the orthogonal matrices of the corresponding dimensions ($\RR^{n\times n}$ and $\RR^{d\times d}$, respectively).

Recall that we denote $\sigma_1^{(n)} \geq \dots \geq \sigma_d^{(n)}$ as the eigenvalues of $\frac{1}{n} \Xb_t^\top \Xb_t$. Thus, we have $\Db_t^\top \Db_t = {\rm Diag}(n\sigma_1^{(n)}, \dots, n\sigma_d^{(n)})$. We can further simplify \eqref{eq:MTMSVD} as
\begin{align}
& \quad~ \frac{1}{n}\EE\left[\Vb_t \left(\Ib_d - (\Db_t^\top \Db_t + n\lambda \Ib_d)^{-1} \Db_t^\top \Db_t \right)^\top \Db_t^\top \Db_t \left(\Ib_d - (\Db_t^\top \Db_t + n\lambda \Ib_d)^{-1} \Db_t^\top \Db_t \right) \Vb_t^\top \right] \notag \\
& = \frac{1}{n} \EE\left[\Vb_t {\rm Diag}\left(\frac{n\lambda^2 \sigma_1^{(n)}}{(\sigma_1^{(n)} + \lambda)^2}, \dots, \frac{n\lambda^2 \sigma_d^{(n)}}{(\sigma_d^{(n)} + \lambda)^2}\right) \Vb_t^\top \right] \label{eq:svdAt} \\
& = \EE\left[\sum_{i=1}^d \frac{\lambda^2 \sigma_i^{(n)}}{(\sigma_i^{(n)} + \lambda)^2} \vb_{t, i} \vb_{t, i}^\top\right].
\label{eq:simplifiedMTM}
\end{align}
We will utilize the isotropicity of $\Xb_t$ to find \eqref{eq:simplifiedMTM}. Recall that we have shown that $\Vb_t$ is uniform on all the orthogonal matrices. Let $\Pb \in \RR^{d \times d}$ be any permutation matrix, then $\Vb_t\Pb$ has the same distribution as $\Vb_t$. For this permuted data matrix $\Vb_t\Pb$, \eqref{eq:simplifiedMTM} becomes
\begin{align}
\EE\left[\sum_{i=1}^d \frac{\lambda^2 \sigma_i^{(n)}}{(\sigma_i^{(n)} + \lambda)^2} \vb_{t, \tau_p(i)} \vb_{t, \tau_p(i)}^\top\right] ~~\textrm{with}~ \tau_p(i) ~\textrm{denotes the permutation of the $i$-th element in}~\Pb. \notag
\end{align}
Summing over all the permutations $\Pb$ (and there are totally $d!$ instances), we deduce
\begin{align}
d! \EE[\Ab_t] & = \sum_{\textrm{all permutation}~\tau_p} \EE\left[\sum_{i=1}^d \frac{\lambda^2 \sigma_i^{(n)}}{(\sigma_i^{(n)} + \lambda)^2} \vb_{t, \tau_p(i)} \vb_{t, \tau_p(i)}^\top\right] \notag \\
& = (d-1)! \EE\left[\sum_{j=1}^d \left[\sum_{i=1}^d \frac{\lambda^2 \sigma_i^{(n)}}{(\sigma_i^{(n)} + \lambda)^2}\right] \vb_{t, j} \vb_{t, j}^\top\right] \notag \\
& = (d-1)! \EE\left[\Vb_t {\rm Diag}\left(\sum_{i=1}^d \frac{\lambda^2 \sigma_i^{(n)}}{(\lambda + \sigma_i^{(n)})^2}, \dots, \sum_{i=1}^d \frac{\lambda^2 \sigma_i^{(n)}}{(\lambda + \sigma_i^{(n)})^2}\right) \Vb_t^\top \right] \notag \\
& = (d-1)! \EE\left[\sum_{i=1}^d \frac{\lambda^2 \sigma_i^{(n)}}{(\lambda + \sigma_i^{(n)})^2} \Vb_t \Vb_t^\top \right].
\label{eq:permutehessian}
\end{align}
Dividing $(d-1)!$ on both sides of \eqref{eq:permutehessian} yields
\begin{align}
\EE[\Ab_t] = \frac{1}{d} \EE \left[\sum_{i=1}^d \frac{\lambda^2 \sigma_i^{(n)}}{(\lambda + \sigma_i^{(n)})^2}\right]\Ib_d. \label{eq:hessianLnosplit}
\end{align}

Similar to the computation of $\Ab_t$, we compute $\EE[\Bb_t]$ as follows.
\begin{align}
\EE[\Bb_t] & = \EE\bigg[\left(\Ib_d - \left((\Xb_t^{\train})^\top \Xb^{\train}_t + n_1\lambda \Ib_d\right)^{-1} (\Xb_t^{\train})^\top \Xb^{\train}_t \right)^\top \frac{(\Xb_t^{\val})^\top \Xb^{\val}_t}{n_2} \notag \\
& \qquad~ \cdot \left(\Ib_d - \left((\Xb^{\train})_t^\top \Xb^{\train}_t + n_1\lambda \Ib_d\right)^{-1} (\Xb_t^{\train})^\top \Xb^{\train}_t \right)\bigg] \notag \\
& \overset{(i)}{=} \EE\bigg[\left(\Ib_d - \left((\Xb_t^{\train})^\top \Xb^{\train}_t + n_1\lambda \Ib_d\right)^{-1} ((\Xb_t^{\train})^\top \Xb^{\train}_t \right)^\top \notag \\
& \qquad~ \cdot \left(\Ib_d - \left((\Xb^{\train})_t^\top \Xb^{\train}_t + n_1\lambda \Ib_d\right)^{-1} (\Xb_t^{\train})^\top \Xb^{\train}_t \right)\bigg] \notag \\
& \overset{(ii)}{=} \EE\Big[\Vb^{\train}_t \left(\Ib_d - ((\Db^{\train}_t)^\top \Db^{\train}_t + n_1\lambda \Ib_d)^{-1} (\Db_t^{\train})^\top \Db^{\train}_t \right)^2 (\Vb_t^{\train})^\top \Big], \label{eq:MTMSVDsplit}
\end{align}
where $(i)$ uses the data generating assumption $\EE[(\Xb_t^{\val})^\top \Xb_t^{\val}] = n_2 \Ib_d$ and the independence between $\Xb_t^{\train}$ and $\Xb_t^{\val}$, and $(ii)$ follows from the SVD of $\Xb_t^{\train} = \Ub_t^{\train} \Db_t^{\train} (\Vb_t^{\train})^\top$.

Here we denote $\sigma_1^{(n_1)} \geq \dots \geq \sigma_d^{(n_1)}$ as the eigenvalues of $\frac{1}{n_1} (\Xb_t^{\train})^\top \Xb_t^{\train}$. Thus, we have $(\Db_t^{\train})^\top \Db^{\train}_t = {\rm Diag}(n_1\sigma_1^{(n_1)}, \dots, n_1\sigma_d^{(n_1)})$. We can now further simplify \eqref{eq:MTMSVDsplit} as
\begin{align}
& \quad \EE\Big[\Vb^{\train}_t \left(\Ib_d - ((\Db^{\train}_t)^\top \Db^{\train}_t + n_1\lambda \Ib_d)^{-1} (\Db_t^{\train})^\top \Db^{\train}_t \right)^2 (\Vb_t^{\train})^\top \Big] \notag \\
& \overset{(i)}{=} \EE\bigg[\Vb_t^{\train} {\rm Diag}\left(\frac{\lambda^2}{(\sigma_1^{(n_1)} + \lambda)^2}, \dots, \frac{\lambda^2}{(\sigma_d^{(n_1)} + \lambda)^2}\right) (\Vb_t^{\train})^\top \bigg] \label{eq:svdbt} \\
& \overset{(ii)}{=} \frac{1}{d} \EE\left[\sum_{i=1}^d \frac{\lambda^2}{(\lambda + \sigma_i^{(n_1)})^2}\right] \Ib_d. \label{eq:hessianLsplit}
\end{align}
Step $(i)$ follows from the same computation in \eqref{eq:svdAt}, and step $(ii)$ uses the permutation trick in \eqref{eq:permutehessian}.

\noindent $\bullet$ {\bf Second moment of $\Ab_t$ and $\Bb_t$}.
Using the SVD of $\Xb_t$ and the computation in \eqref{eq:svdAt}, we derive
\begin{align}
\EE [\Ab_t^2] & = \EE\bigg[\Vb_t {\rm Diag}\left(\frac{\lambda^2 \sigma_1^{(n)}}{(\sigma_1^{(n)} + \lambda)^2}, \dots, \frac{\lambda^2 \sigma_d^{(n)}}{(\sigma_d^{(n)} + \lambda)^2}\right) \Vb_t^\top \cdot \Vb_t {\rm Diag}\left(\frac{\lambda^2 \sigma_1^{(n)}}{(\sigma_1^{(n)} + \lambda)^2}, \dots, \frac{\lambda^2 \sigma_d^{(n)}}{(\sigma_d^{(n)} + \lambda)^2}\right) \Vb_t^\top \bigg] \notag \\
& = \EE\left[\Vb_t {\rm Diag}\left(\frac{\lambda^4 (\sigma_1^{(n)})^2}{(\sigma_1^{(n)} + \lambda)^4}, \dots, \frac{\lambda^4 (\sigma_d^{(n)})^2}{(\sigma_d^{(n)} + \lambda)^4}\right) \Vb_t^\top\right] \notag \\
& \overset{(i)}{=} \frac{1}{d} \EE\left[{\rm Diag}\left(\frac{\lambda^4 (\sigma_1^{(n)})^2}{(\sigma_1^{(n)} + \lambda)^4}, \dots, \frac{\lambda^4 (\sigma_d^{(n)})^2}{(\sigma_d^{(n)} + \lambda)^4}\right)\right]  \Ib_d, \nonumber
\end{align}
where step $(i)$ applies the permutation trick in \eqref{eq:permutehessian} to tackle $\EE\left[\Vb_t {\rm Diag}\left(\frac{\lambda^4 (\sigma_1^{(n)})^2}{(\sigma_1^{(n)} + \lambda)^4}, \dots, \frac{\lambda^4 (\sigma_d^{(n)})^2}{(\sigma_d^{(n)} + \lambda)^4}\right) \Vb_t^\top\right]$.

For $\EE[\Bb_t^2]$, the computation is a bit more complex. Using the SVD of $\Xb_t^{\train}$ as in \eqref{eq:svdbt}, we obtain
\begin{align}
\EE[\Bb_t^2] & = \frac{1}{n_2^2} \EE\bigg[\Vb_t^{\train} {\rm Diag}\left(\frac{\lambda}{\sigma_1^{(n_1)} + \lambda}, \dots, \frac{\lambda}{\sigma_d^{(n_1)} + \lambda}\right) (\Vb_t^{\train})^\top (\Xb_t^{\val})^\top \notag \\
& \qquad~ \cdot \Xb_t^{\val} \Vb_t^{\train} {\rm Diag}\left(\frac{\lambda^2}{(\sigma_1^{(n_1)} + \lambda)^2}, \dots, \frac{\lambda^2}{(\sigma_d^{(n_1)} + \lambda)^2}\right) (\Vb_t^{\train})^\top (\Xb_t^{\val})^\top \notag \\
& \qquad~ \cdot \Xb_t^{\val} \Vb_t^{\train} {\rm Diag}\left(\frac{\lambda}{\sigma_1^{(n_1)} + \lambda}, \dots, \frac{\lambda}{\sigma_d^{(n_1)} + \lambda}\right) (\Vb_t^{\train})^\top \bigg]. \label{eq:diagonalbt2}
\end{align}
We claim that $\EE[\Bb_t^2]$ is diagonal. To see this, we take expectation with respect to $\Xb_t^{\val}$ first in \eqref{eq:diagonalbt2}. Since $\Vb_t^{\train}$ is an orthogonal matrix, $\Xb_t^{\val}\Vb_t^{\train}$ has the same distribution as $\Xb_t^{\val}$ and independent of $\Xb_t$. We verify that any off-diagonal element is zero in the following matrix
\begin{align}
\Tb \defeq ~& \EE_{\Xb_t^{\val}} \bigg[(\Vb_t^{\train})^\top (\Xb_t^{\val})^\top \Xb_t^{\val} \Vb_t^{\train} {\rm Diag}\left(\frac{\lambda^2}{(\sigma_1^{(n_1)} + \lambda)^2}, \dots, \frac{\lambda^2}{(\sigma_d^{(n_1)} + \lambda)^2}\right) \notag \\
& \qquad~ \cdot (\Vb_t^{\train})^\top (\Xb_t^{\val})^\top\Xb_t^{\val} \Vb_t^{\train}\bigg]. \notag
\end{align} 
We denote $\Xb_t^{\val}\Vb_t^{\train} = [\xb_1, \dots, \xb_n]^\top \in \RR^{n_2 \times d}$ with $\xb_i \overset{\rm iid}{\sim} \normal(\bzero, \Ib_d)$. For $k \neq \ell$, the $(k, \ell)$-th entry $T_{k, \ell}$ of $\Tb$ is
\begin{align}
T_{k, \ell} & = \EE\left[\sum_j \left(\frac{\lambda^2}{(\sigma_j^{(n_1)} + \lambda)^2} \left(\sum_i x_{k, i} x_{j, i}\right) \left(\sum_i x_{j, i} x_{\ell, i}\right)\right)\right] \notag \\
& = \EE\left[\sum_j \frac{\lambda^2}{(\sigma_j^{(n_1)} + \lambda)^2} \left(\sum_{m, n} x_{k, m} x_{j, m} x_{j, n} x_{\ell, n}\right)\right] \notag \\
& \overset{(i)}{=} 0, \notag
\end{align}
where $x_{i, j}$ denotes the $j$-th element of $\xb_i$. Equality $(i)$ holds, since either $x_{k, m}$ or $x_{\ell, n}$ only appears once in each summand. Therefore, we can write $\Tb = {\rm Diag}\left(T_{1, 1}, \dots, T_{d, d}\right)$ with $T_{k, k}$ being
\begin{align}
T_{k, k} & = \EE\left[\sum_j \frac{\lambda^2}{(\sigma_j^{(n_1)} + \lambda)^2} \left(\sum_{m, n} x_{k, m} x_{j, m} x_{j, n} x_{\ell, n}\right)\right] \notag \\
& = \EE\left[\frac{\lambda^2}{(\sigma_k^{(n_1)} + \lambda)^2} \left(\sum_{m, n} x_{k, m} x_{k, m} x_{k, n} x_{k, n}\right)\right]. \notag
\end{align}
Observe that $T_{k, k}$ only depends on $\sigma_k^{(n_1)}$. Plugging back into \eqref{eq:diagonalbt2}, we have
\begin{align}
\EE[\Bb_t^2] & = \frac{1}{n_2^2} \EE\bigg[\Vb_t^{\train} {\rm Diag}\left(\frac{\lambda}{\sigma_1^{(n_1)} + \lambda}, \dots, \frac{\lambda}{\sigma_d^{(n_1)} + \lambda}\right) (\Vb_t^{\train})^\top (\Xb_t^{\val})^\top \notag \\
& \qquad~ \cdot \Xb_t^{\val} \Vb_t^{\train} {\rm Diag}\left(\frac{\lambda^2}{(\sigma_1^{(n_1)} + \lambda)^2}, \dots, \frac{\lambda^2}{(\sigma_d^{(n_1)} + \lambda)^2}\right) (\Vb_t^{\train})^\top (\Xb_t^{\val})^\top \notag \\
& \qquad~ \cdot \Xb_t^{\val} \Vb_t^{\train} {\rm Diag}\left(\frac{\lambda}{\sigma_1^{(n_1)} + \lambda}, \dots, \frac{\lambda}{\sigma_d^{(n_1)} + \lambda}\right) (\Vb_t^{\train})^\top \bigg] \notag \\
& = \frac{1}{n_2^2} \EE\bigg[\Vb_t^{\train} {\rm Diag}\left(\frac{\lambda}{\sigma_1^{(n_1)} + \lambda}, \dots, \frac{\lambda}{\sigma_d^{(n_1)} + \lambda}\right) {\rm Diag}(T_{1, 1}, \dots, T_{d,d}) \notag \\
& \qquad~ \cdot {\rm Diag}\left(\frac{\lambda}{\sigma_1^{(n_1)} + \lambda}, \dots, \frac{\lambda}{\sigma_d^{(n_1)} + \lambda}\right) (\Vb_t^{\train})^\top \bigg] \notag \\
& = \frac{1}{n_2^2} \EE\bigg[\Vb_t^{\train} {\rm Diag}\left(\frac{\lambda^2 T_{1, 1}}{(\sigma_1^{(n_1)} + \lambda)^2}, \dots, \frac{\lambda^2 T_{d, d}}{(\sigma_d^{(n_1)} + \lambda)^2}\right) (\Vb_t^{\train})^\top \bigg] \notag \\
& \overset{(i)}{=} c \Ib_d, \label{eq:bt2c}
\end{align}
where equality $(i)$ utilizes the permutation trick in \eqref{eq:hessianLnosplit}. To this end, it suffices to find $c$ as
\begin{align}
c & = \frac{1}{d} \EE[\Bb_t^2] \nonumber \\
& = \frac{1}{dn_2^2} {\rm tr}\bigg(\EE\bigg[\Vb_t^{\train} {\rm Diag}\left(\frac{\lambda}{\sigma_1^{(n_1)} + \lambda}, \dots, \frac{\lambda}{\sigma_d^{(n_1)} + \lambda}\right) (\Vb_t^{\train})^\top (\Xb_t^{\val})^\top \notag \\
& \qquad~ \cdot \Xb_t^{\val} \Vb_t^{\train} {\rm Diag}\left(\frac{\lambda^2}{(\sigma_1^{(n_1)} + \lambda)^2}, \dots, \frac{\lambda^2}{(\sigma_d^{(n_1)} + \lambda)^2}\right) (\Vb_t^{\train})^\top (\Xb_t^{\val})^\top \notag \\
& \qquad~ \cdot \Xb_t^{\val} \Vb_t^{\train} {\rm Diag}\left(\frac{\lambda}{\sigma_1^{(n_1)} + \lambda}, \dots, \frac{\lambda}{\sigma_d^{(n_1)} + \lambda}\right) (\Vb_t^{\train})^\top \bigg]\bigg) \notag \\
& = \frac{1}{dn_2^2} {\rm tr}\bigg(\EE\bigg[\Xb_t^{\val} \Vb_t^{\train} {\rm Diag}\left(\frac{\lambda^2}{(\sigma_1^{(n_1)} + \lambda)^2}, \dots, \frac{\lambda^2}{(\sigma_d^{(n_1)} + \lambda)^2}\right) (\Vb_t^{\train})^\top (\Xb_t^{\val})^\top \notag \\
& \qquad~ \cdot \Xb_t^{\val} \Vb_t^{\train} {\rm Diag}\left(\frac{\lambda^2}{(\sigma_1^{(n_1)} + \lambda)^2}, \dots, \frac{\lambda^2}{(\sigma_d^{(n_1)} + \lambda)^2}\right) (\Vb_t^{\train})^\top (\Xb_t^{\val})^\top \bigg]\bigg). \label{eq:squaretrace}
\end{align}
Observe again that $\Xb_t^{\val} \Vb_t^{\train} \in \RR^{n_2 \times d}$ is a Gaussian random matrix. We rewrite \eqref{eq:squaretrace} as
\begin{align}
c & = \frac{1}{dn_2^2} \EE\left[\left(\sum_{i, j=1}^{n_2} \vb_i^\top {\rm Diag}\left(\frac{\lambda^2}{(\sigma_1^{(n_1)} + \lambda)^2}, \dots, \frac{\lambda^2}{(\sigma_d^{(n_1)} + \lambda)^2}\right) \vb_j\right)^2 \right], \label{eq:simplifiedsquaretrace}
\end{align}
where $\vb_i \overset{\rm iid}{\sim} \normal(\bzero, \Ib_d)$ is i.i.d. Gaussian random vectors for $i = 1, \dots, n_2$. To compute \eqref{eq:simplifiedsquaretrace}, we need the following result.
\begin{claim}\label{claim}
Given any symmetric matrix $\Tb \in \RR^{d \times d}$ and i.i.d. standard Gaussian random vectors $\vb, \ub \overset{\rm iid}{\sim} \normal(\bzero, \Ib_d)$, we have
\begin{align}
\EE \left[(\vb^\top \Tb \vb)^2\right] & = 2 \|\Tb\|_{\sf Fr}^2 + {\rm tr}^2(\Tb)\quad \textrm{and} \label{eq:claimresult1} \\
\EE \left[(\vb^\top \Tb \ub)^2 \right] & = \|\Tb\|_{\sf Fr}^2. \label{eq:claimresult2}
\end{align}
\end{claim}
\begin{proof}[Proof of Claim \ref{claim}]
We show \eqref{eq:claimresult1} first. We denote $T_{i, j}$ as the $(i, j)$-th element of $\Tb$ and $v_i$ as the $i$-th element of $\vb$. Expanding the quadratic form, we have
\begin{align}
\EE \left[(\vb^\top \Tb \vb)^2\right] & = \EE\left[\sum_{i, j, k, \ell \leq d} v_i v_j v_k v_{\ell} T_{i, j} T_{k, \ell} \right] \notag \\
& = \EE\left[\sum_{i \leq d} v_i^4 T_{i, i}^2 \right] + \EE \left[\sum_{i \neq j} v_i^2 v_j^2 (T_{i, j}^2 + T_{i, i} T_{j, j} + T_{i, j} T_{j, i})\right] \notag \\
& = 3 \sum_{i \leq d} T_{i, i}^2 + \sum_{i \neq j} (T_{i, j}^2 + T_{i, i} T_{j, j} + T_{i, j} T_{j, i}) \notag \\
& = {\rm tr}^2 (\Tb) + 2 \sum_{i \leq d} T_{i, i}^2 + \sum_{i \neq j} (T_{i, j}^2 + T_{i, j} T_{j, i}) \notag \\
& = {\rm tr}^2 (\Tb) + 2 \|\Tb\|_{\sf Fr}^2.\notag
\end{align}
Next, we show \eqref{eq:claimresult2} by the cyclic property of race.
\begin{align}
\EE \left[(\vb^\top \Tb \ub)^2 \right] = {\rm tr}\left(\EE\left[\ub\ub^\top \Tb \vb \vb^\top \Tb \right]\right) = {\rm tr}(\Tb^2) = \|\Tb\|_{\sf Fr}^2. \notag
\end{align}
\end{proof}
We back to the computation of \eqref{eq:simplifiedsquaretrace} using Claim \ref{claim}.
\begin{align}
c & = \frac{1}{dn_2^2} \EE\left[\sum_{i, j=1}^{n_2} \left(\vb_i^\top {\rm Diag}\left(\frac{\lambda^2}{(\sigma_1^{(n_1)} + \lambda)^2}, \dots, \frac{\lambda^2}{(\sigma_d^{(n_1)} + \lambda)^2}\right) \vb_j\right)^2 \right] \notag \\
& = \frac{1}{dn_2^2} \EE \left[\sum_{i=1}^{n_2} \left(\vb_i^\top {\rm Diag}\left(\frac{\lambda^2}{(\sigma_1^{(n_1)} + \lambda)^2}, \dots, \frac{\lambda^2}{(\sigma_d^{(n_1)} + \lambda)^2}\right) \vb_i\right)^2 \right] \notag \\
& \qquad + \frac{1}{dn_2^2} \EE \left[\sum_{i\neq j} \left(\vb_i^\top {\rm Diag}\left(\frac{\lambda^2}{(\sigma_1^{(n_1)} + \lambda)^2}, \dots, \frac{\lambda^2}{(\sigma_d^{(n_1)} + \lambda)^2}\right) \vb_j\right)^2 \right] \notag \\
& = \frac{1}{dn_2} \EE \left[{\rm tr}^2 \left({\rm Diag}\left(\frac{\lambda^2}{(\sigma_1^{(n_1)} + \lambda)^2}, \dots, \frac{\lambda^2}{(\sigma_d^{(n_1)} + \lambda)^2}\right)\right)\right] \notag \\
& \qquad + \frac{2}{dn_2} \EE\left[\left\|{\rm Diag}\left(\frac{\lambda^2}{(\sigma_1^{(n_1)} + \lambda)^2}, \dots, \frac{\lambda^2}{(\sigma_d^{(n_1)} + \lambda)^2}\right)\right\|_{\sf Fr}^2\right] \notag \\
& \qquad + \frac{n_2-1}{dn_2} \EE \left[\left\|{\rm Diag}\left(\frac{\lambda^2}{(\sigma_1^{(n_1)} + \lambda)^2}, \dots, \frac{\lambda^2}{(\sigma_d^{(n_1)} + \lambda)^2}\right)\right\|_{\sf Fr}^2\right] \notag \\
& = \frac{1}{dn_2} \left(\EE \left[\sum_{i=1}^d \frac{\lambda^2}{(\sigma_i^{(n_1)} + \lambda)^2} \right]^2 + (n_2+1) \EE\left[\sum_{i=1}^d \frac{\lambda^4}{(\sigma_i^{(n_1)} + \lambda)^4}\right]\right). \label{eq:cvalue}
\end{align}
Substituting the value of $c$ in \eqref{eq:cvalue} into \eqref{eq:bt2c}, we derive the desired result
\begin{align}
\EE[\Bb_t^2] = \frac{1}{dn_2} \left(\EE \left[\sum_{i=1}^d \frac{\lambda^2}{(\sigma_i^{(n_1)} + \lambda)^2} \right]^2 + (n_2+1) \EE\left[\sum_{i=1}^d \frac{\lambda^4}{(\sigma_i^{(n_1)} + \lambda)^4}\right]\right). \notag
\end{align}
The proof is complete.
\end{proof}

\begin{lemma}[Bounds on moments]
  \label{lemma:moments-bounds}
  Let $f_{\set{A,A^2}}(n,d,\lambda)$ and $f_{\set{B,B^2}}(n_1, n_2, d, \lambda)$ be defined as in Lemma~\ref{lemma:moments}. Suppose $d/n=\gamma=\Theta(1)$ and $n_1/n=s=\Theta(1)$, $\lambda=\Theta(1)>0$. Then we have
  \begin{align*}
    \underline{c}_A\le f_A(n,d,\lambda), f_{A^2}(n,d,\lambda) \le \overline{c}_A,
  \end{align*}
  and
  \begin{align*}
    \underline{c}_B \le f_B(n_1,n_2,d,\lambda), f_{B^2}(n_1, n_2, d, \lambda) \le \overline{c}_B,
  \end{align*}
  where $\underline{c}_A,\overline{c}_A>0$ depend only on $\gamma,\lambda$ but not $d$, and $\underline{c}_B, \overline{c}_B>0$ depend only on $\gamma,s,\lambda$ but not $d$.
\end{lemma}
\begin{proof}
  The upper bounds follow straightforwardly from the closed-form expressions established in Lemma~\ref{lemma:moments}: we have
  \begin{align*}
    & f_A(n,d,\lambda) = \frac{1}{d} \E\brac{ \tr\paren{ \lambda^2 \paren{\hat{\bSigma}_n + \lambda\Ib_d}^{-2}\hat{\bSigma}_n} } \le \frac{1}{d} \E\brac{ \tr\paren{ \hat{\bSigma}_n} } = 1. \\
    & f_{A^2}(n,d,\lambda) = \frac{1}{d} \E\brac{ \tr\paren{\lambda^4\paren{ \hat{\bSigma}_n + \lambda\Ib_d }^{-2}\hat{\bSigma}_n^2} } \le \frac{1}{d}\E\brac{ \tr\paren{\hat{\bSigma}_n^2} } \stackrel{(i)}{=} \frac{1}{d} \cdot \paren{\frac{d^2+(n+1)d}{n}} = \frac{d+n+1}{n} \le 2 + \gamma,
  \end{align*}
where (i) used the fact that
\begin{align*}
  \E\brac{ \tr\paren{\hat{\bSigma}_n^2} } = \E\brac{ \lfro{\hat{\bSigma}_n}^2 } = \frac{1}{n^2}\cdot \E\brac{ \sum_{i,j=1}^n (\xb_i^\top\xb_j)^2 } = \frac{1}{n^2} \cdot \brac{ n(n-1)d + n(d^2+2d) } = \frac{d^2+(n+1)d}{n}.
\end{align*}
Therefore, we can take $\overline{c}_A=2+\gamma$. Similarly, we have
\begin{align*}
  & f_B(n_1, n_2, d, \lambda) = \frac{1}{d}\E\brac{ \sum_{i=1}^d \lambda^2 / (\sigma_i^{(n_1)} + \lambda)^2 } \le 1, \\
  & f_{B^2}(n_1, n_2, d, \lambda) = \frac{1}{dn_2}  \E\brac{ \paren{ \sum_{i=1}^d \lambda^2 / (\sigma_i^{(n_1)} + \lambda)^2 }^2 + (n_2+1) \sum_{i=1}^d \lambda^4/(\sigma_i^{(n_1)} + \lambda)^4 } \\
  & \qquad \le \frac{1}{dn_2} \brac{ d^2 + d(n_2+1) } = \frac{d+n_2+1}{n_2} = \frac{d + n(1-s) + 1}{n(1-s)} \le \frac{\gamma+2-s}{1-s}.
\end{align*}
Therefore we can take $\overline{c}_B=(2+\gamma-s)/(1-s)$.

For the lower bounds, it suffices to prove the lower bounds for $f_A(n,d,\lambda)$ and $f_B(n_1, n_2, d, \lambda)$ (as $\frac{1}{d}\tr(\Mb^2)\ge (\frac{1}{d}\tr(\Mb))^2$ always holds for any PSD matrix $\Mb\in\R^{d\times d}$). For this we apply the same Stieltjes calculation as in the proof of Theorem~\ref{theorem:high-dim-limit} to conclude that
\begin{align*}
  \lim_{d,n\to\infty, d/n\to\gamma} f_A(n,d,\lambda) = \frac{\lambda^2}{2\gamma} \paren{ \frac{\lambda+1+\gamma}{\sqrt{(\lambda+1+\gamma)^2-4\gamma}} - 1} > 0.
\end{align*}
Also note that $f_A(\ceil{d/\gamma}, d, \lambda) > 0$ for any $d\ge 1$. Therefore, we have
\begin{align*}
  \inf_{d\ge 1} f_A(\ceil{d/\gamma}, d, \lambda) \defeq \wt{c}_A > 0.
\end{align*}
Taking $\underline{c}_A=\min\set{\wt{c}_A, \wt{c}_A^2}>0$ (which only depends on $\gamma$, we get $\min\set{f_A(n,d,\lambda), f_{A^2}(n,d,\lambda)} \ge \underline{c}_A$, the desired result. Similarly, for $f_B$, we have
\begin{align*}
  & \quad \lim_{d,n_1\to\infty, d/n_1\to\gamma/s} \frac{1}{d} \E\brac{ \tr\paren{ \lambda^2 (\lambda\Ib_d + \hat{\bSigma}_{n_1})^{-2}} } = \lambda^2 \cdot \brac{-\frac{d}{d\lambda_1} s(\lambda_1, \lambda_2)|_{\lambda_1=\lambda, \lambda_2=1}} , \\
  & = \lambda^2 \cdot \frac{1}{4\gamma/s\cdot \lambda^2\sqrt{(\lambda+1+\gamma/s)^2-4\gamma/s}} \brac{ 2(\gamma-1)\sqrt{(\lambda+1+\gamma/s)^2-4\gamma/s} + 2\lambda(1+\gamma) + 2(1-\gamma)^2} > 0,
\end{align*}
where $s(\lambda_1, \lambda_2)$ is the generalized Stieltjes transform defined in~\eqref{equation:stj-def}. (The detailed calculation can be found in Section~\ref{section:calculation-stj-dlambda1}.) As this limit is strictly positive, ssing a similar argument as the above, we get that there exists some $\underline{c}_B>0$ which only depends on $\gamma/s=\Theta(1)>0$, such that $\min\set{f_B(n_1,n_2,d,\lambda), f_{B^2}(n_1, n_2,d,\lambda)}\ge \underline{c}_B$.
\end{proof}

\subsubsection{Calculations of $\frac{d}{d\lambda_1} s(\lambda_1, \lambda_2)$}
\label{section:calculation-stj-dlambda1}
Recall by~\eqref{equation:stj-def} that
\begin{align*}
  s(\lambda_1, \lambda_2) = \frac{\gamma-1-\lambda_1/\lambda_2 + \sqrt{(\lambda_1/\lambda_2+1+\gamma)^2-4\gamma}}{2\gamma\lambda_1}.
\end{align*}
At $\lambda_1=\lambda$ and $\lambda_2=1$, the above can be simplied as
\begin{align*}
  s(\lambda, 1) = \frac{\gamma-1}{2\gamma\lambda} - \frac{1}{2\gamma} + \frac{1}{2\gamma}\sqrt{(1 + (1+\gamma)/\lambda)^2 - 4\gamma/\lambda^2}.
\end{align*}
Differentiating with respect to $\lambda$, we get
\begin{align*}
  & \quad -\frac{d}{d\lambda_1} s(\lambda_1, \lambda_2)|_{\lambda_1=\lambda, \lambda_2=1} = -\frac{d}{d\lambda} s(\lambda, 1) = \frac{\gamma-1}{2\gamma\lambda^2} + \frac{(\lambda+1+\gamma)(1+\gamma) - 4\gamma}{2\gamma\lambda^2\sqrt{(\lambda+1+\gamma)^2 - 4\gamma}} \\
  & = \frac{1}{2\gamma\lambda^2\sqrt{(\lambda+1+\gamma)^2 - 4\gamma}} \brac{ (\gamma-1)\sqrt{(\lambda+1+\gamma)^2 - 4\gamma} +\lambda(1+\gamma) + (\gamma-1)^2 }.
\end{align*}
The above is clearly positive at $\gamma\ge 1$. At $\gamma<1$, we have
\begin{align*}
  & \quad  (1-\gamma)\sqrt{(\lambda+1+\gamma)^2-4\gamma} = \sqrt{\lambda^2 + 2\lambda(1+\gamma) + (1-\gamma)^2} \\
  & < (1-\gamma) \paren{ \sqrt{\paren{\lambda(1+\gamma)/(1-\gamma)}^2 + 2\lambda(1+\gamma) + (1-\gamma)^2} }= (1-\gamma)\paren{ \lambda(1+\gamma)/(1-\gamma) + 1-\gamma} \\
  & = \lambda(1+\gamma) + (1-\gamma)^2.
\end{align*}
Therefore, the quantity inside the bracket in the preceding display is also strictly positive. This shows that $-\frac{d}{d\lambda_1}s(\lambda_1, \lambda_2)|_{\lambda_1=\lambda,\lambda_2=1}>0$ for all $\gamma>0$.

\subsection{Concentration lemmas}
\label{appendix:concentration-lemmas}

\begin{lemma}[Concentration of $\Ab_t$ and $\Bb_t$]
  \label{lemma:concentration-atbt}
  Let $\Ab_t$ and $\Bb_t$ be defined as in~\eqref{equation:at} and~\eqref{equation:bt}. Then with probability at least $1-Td^{-10}$, we have the following bounds: $\bzero \preceq \Ab_t \preceq \wt{O}(C_a)\Ib_d$ and $\bzero \preceq \Bb_t \preceq \wt{O}(C_b)\Ib_d$, and
  \begin{align*}
        & \opnorm{ \frac{1}{T}\sum_{t=1}^T \Ab_t - \E\brac{\Ab_t} } \le \wt{O}\paren{ C_a\sqrt{\frac{d}{T}} + d^{-4} }~~~{\rm and}~~~\opnorm{ \frac{1}{T}\sum_{t=1}^T \Ab_t^2 - \E\brac{\Ab_t^2} } \le \wt{O}\paren{ C_a^2\sqrt{\frac{d}{T}} + d^{-4}}, \\
    & \opnorm{ \frac{1}{T}\sum_{t=1}^T \Bb_t - \E\brac{\Bb_t} } \le \wt{O}\paren{ C_b\sqrt{\frac{d}{T}} + d^{-4} }~~~{\rm and}~~~\opnorm{ \frac{1}{T}\sum_{t=1}^T \Bb_t^2 - \E\brac{\Bb_t^2} } \le \wt{O}\paren{ C_b^2\sqrt{\frac{d}{T}} + d^{-4}} ,
  \end{align*}
  where $C_a\defeq 1 + \max\set{d/n, \sqrt{d/n}}$, $C_b\defeq 1 + \max\set{d/n_2, \sqrt{d/n_2}}$, and $\wt{O}(\cdot)$ hides the logarithmic factor $\log(ndT)$.  
\end{lemma}
\begin{proof}  
  We first prove the result for $\Bb_t$. We use a truncation argument. Recall that by definition of $\Bb_t$ we have
  \begin{align*}
    \bzero \preceq \Bb_t \preceq \frac{\Xb_t^{\val\top}\Xb_t^{\val}}{n_2}.
  \end{align*}
  Since $\xb_{t,i}\sim\normal(\bzero, \Ib_d)$, applying the standard sub-Gaussian covariance concentration~\citep[Exercise 4.7.3]{vershynin2018high}, we have with probability at least $1-d^{-10}$ that
  \begin{align*}
    \Bb_t \preceq \frac{1}{n_2}\Xb_t^{\val\top}\Xb_t^{\val} \preceq \Ib_d + \opnorm{\frac{1}{n_2}\Xb_t^{\val\top}\Xb_t^{\val} - \Ib_d  }\Ib_d \preceq \paren{1 + C\sqrt{\frac{d + \log d}{n_2}} + C\frac{d + \log d}{n_2}} \Ib_d \preceq KC_b \Ib_d,
  \end{align*}
  where $C_b\defeq1 + \max\set{d/n_2, \sqrt{d/n_2}}$ and $K=O(1)$ is an absolute constant.
  Let $\Eb_t\defeq \set{ \Bb_t \preceq KC_b\Ib_d }$ denote this event. We have $\P(\Eb_t) \ge 1-d^{-10}$. Let $\Eb\defeq \bigcup_{t=1}^T \Eb_t$ denote the union event. Note that on the event $\Eb$ we have
  \begin{align*}
    \frac{1}{T}\sum_{t=1}^T \Bb_t = \frac{1}{T}\sum_{t=1}^T \Bb_t \indic{\Eb_t} .
  \end{align*}

  \paragraph{Concentration of $\Bb_t\indic{\Eb_t}$}
  On the event $\Eb_t$, $\Bb_t$ are bounded matrices:
  \begin{align*}
    \bzero \preceq \Bb_t\indic{\Eb_t} \preceq C_b \Ib_d.
  \end{align*}
  In particular, this means that for any $\vb\in\R^d$ with unit norm $\norm{\vb}=1$, the random variable
  \begin{align*}
    \vb^\top \Bb_t\indic{\Eb_t}  \vb - \vb^\top \E[\Bb_t\indic{\Eb_t} ] \vb
  \end{align*}
  is mean-zero and $C^2$-sub-Gaussian. Therefore by the standard sub-Gaussian concentration, we get
  \begin{align*}
    \P\paren{ \abs{ \vb^\top \paren{\frac{1}{T}\sum_{t=1}^T \Bb_t\indic{\Eb_t} } \vb - \vb^\top \E\brac{\Bb_t\indic{\Eb_t} }\vb} \ge t } \le 2\exp\paren{-Tt^2/C_b^2}.
  \end{align*}
  Using the fact that for any symmetric matrix $\Mb$,
  \begin{align*}
    \opnorm{\Mb} \le 2\sup_{\vb\in N_{1/4}(\mathbb{S}^{d-1})} \abs{\vb^\top \Mb \vb},
  \end{align*}
  where $N_{1/4}(\mathbb{S}^{d-1})$ is a $1/4$-covering set of the unit sphere with $|N_{1/4}(\mathbb{S}^{d-1})|\le 9^d$~\citep[Exercise 4.4.3]{vershynin2018high}, we get
  \begin{align*}
    & \quad \P\paren{ \opnorm{ \frac{1}{T}\sum_{t=1}^T \Bb_t\indic{\Eb_t}  - \E\brac{\Bb_t\indic{\Eb_t} } } \ge t } \\
    & \le \abs{ N_{1/4}(\mathbb{S}^{d-1}) } \cdot \sup_{\norm{\vb}=1} \P\paren{ \abs{ \vb^\top \paren{\frac{1}{T}\sum_{t=1}^T \Bb_t\indic{\Eb_t} } \vb - \vb^\top \E\brac{\Bb_t\indic{\Eb_t} }\vb} \ge t } \\
    & \le \exp\paren{ -Tt^2/C_b^2 + 3d }.
  \end{align*}
  Taking $t=O(C_b\sqrt{\frac{d + \log(1/d^{10})}{T}})=\wt{O}(C_b\sqrt{d/T})$, the above probability is upper bounded by $d^{-10}/2$. In other words, with probability at least $1-d^{-10}/2$, we get
  \begin{align*}
    \opnorm{ \frac{1}{T}\sum_{t=1}^T \Bb_t\indic{\Eb_t}  - \E\brac{\Bb_t\indic{\Eb_t} } } \le \wt{O}\paren{ C_b\sqrt{\frac{d}{T}} }.
  \end{align*}

  \paragraph{Bounding difference betwen $\E[\Bb_t]$ and $\E[\Bb_t\indic{\Eb_t}]$}
  We have
  \begin{align*}
    & \opnorm{\E[\Bb_t] - \E[\Bb_t\indic{\Eb_t}]} \le \E\brac{ \opnorm{\Bb_t} \indic{\Eb_t^c} } \le \paren{ \E\brac{ \opnorm{\Bb_t}^2 } \cdot \P(\Eb_t^c)}^{1/2} \\
    & \le \sqrt{ \E\brac{ \max_i \norm{\xb_{t,i}}^2 } \cdot d^{-10} } \le \sqrt{({d + C\log n_2) \cdot d^{-10} }} = \wt{O}(d^{-4.5}).
  \end{align*}
  where the last inequality is by standard Gaussian norm concentration (e.g.~\citep[Appendix A.3]{bai2019beyond}). 
  
  \paragraph{Concentration of $\Bb_t$}
  Combining the preceding two parts, we get that with probability at least $1-Td^{-10}$ that
  \begin{align*}
    & \quad \opnorm{\frac{1}{T}\sum_{t\le T} \Bb_t - \E[\Bb_t]} \\
    & \le \opnorm{\frac{1}{T}\sum_{t\le T} \Bb_t\indic{\Eb_t} - \E[\Bb_t\indic{\Eb_t}]} + \opnorm{\E[\Bb_t] - \E[\Bb_t\indic{\Eb_t}]} \le \wt{O}\paren{C_b\sqrt{d/T} + d^{-4.5}}.
  \end{align*}
  
  \paragraph{Concentration for $\Bb_t^2$, $\Ab_t$, and $\Ab_t^2$}
  For $\Bb_t^2$, using a similar analysis as the above, we get
  \begin{align*}
    \opnorm{ \frac{1}{T}\sum_{t=1}^T \Bb_t^2 - \E\brac{\Bb_t^2} } \le \wt{O}\paren{ C_b^2\sqrt{\frac{d}{T}} + d^{-4}}.
  \end{align*}
  For $\Ab_t$, we note that the bound
  \begin{align*}
    \bzero \preceq \Ab_t \preceq C_a\Ib_d,~~~{\rm where}~C_a=1 + O\paren{\max\set{d/n, \sqrt{d/n}}}
  \end{align*}
  holds. Therefore using the same argument as above, we get the desired concentration bounds for $\Ab_t$.
\end{proof}

We also need the following Hanson-Wright inequality.
\begin{lemma}[Restatement of Theorem 6.2.1,~\citep{vershynin2018high}]
  \label{lemma:hanson-wright}
  Let $\zb\in\R^D$ be a random vector with independent, mean-zero, and $O(K^2)$-sub-Gaussian entries, and let $\Cb\in\R^{D\times D}$ be a fixed matrix. Then we have with probability at least $1-\delta$ that
  \begin{align*}
    \abs{ \zb^\top \Cb\zb - \E\brac{\zb^\top \Cb\zb} } \le O\paren{  K^2 \max\set{\lfro{\Cb}\sqrt{\log(2/\delta)}, \opnorm{\Cb}\log(2/\delta)} } \le O\paren{K^2\lfro{\Cb}\log(2/\delta)}.
  \end{align*}
\end{lemma}

\subsection{Proof of main theorem}
\label{appendix:proof-main-concentration-mse}

We are now ready to prove Theorem~\ref{theorem:concentration-mse}. We first prove the result for $\hat{\wb}^{\nosp}_{0,T}$. Define the matrix
\begin{align*}
  \bSigma_T \defeq \paren{ \sum_{t=1}^T \Ab_t }^{-2} \sum_{t=1}^T \Ab_t^2 = \frac{1}{T}\cdot \paren{ \frac{\sum_{t=1}^T \Ab_t}{T} }^{-2} \frac{\sum_{t=1}^T \Ab_t^2}{T},
\end{align*}
which will be key to our analysis. Observe that
\begin{align*}
  \hat{\wb}_{0,T}^{\nosp} - \wb_{0,\star} = \paren{\sum_{t=1}^T \Ab_t}^{-1} \sum_{t=1}^T \Ab_t(\wb_t - \wb_{0,\star}),
\end{align*}
Therefore, conditioned on $\Ab_t$ (and only looking at the randomness of $\wb_t$), we have
\begin{align*}
  & \quad \E_{\wb_t}\brac{ {\rm MSE}(\hat{\wb}_{0,T}^{\nosp}) } = \E_{\wb_t}\brac{ \norm{\hat{\wb}_{0,T}^{\nosp} - \wb_{0,\star}}^2 } \\
  & = \sum_{t=1}^T \tr\paren{ \paren{ \sum_{t=1}^T \Ab_t}^{-1} \Ab_t \cdot \Cov(\wb_t) \Ab_t^\top \paren{ \sum_{t=1}^T \Ab_t}^{-1} } = \frac{R^2}{d} \tr\paren{ \paren{ \sum_{t=1}^T \Ab_t}^{-2} \sum_{t=1}^T \Ab_t^2} = \frac{R^2}{d}\tr(\bSigma_T).
\end{align*}

\paragraph{Concentration of $\Ab_t$}
By Lemma~\ref{lemma:moments-bounds} and Lemma~\ref{lemma:concentration-atbt}, we have with probability at least $1-Td^{-10}$ that
\begin{align*}
  \underline{c}_A\Ib_d \preceq \E[\Ab_t] \preceq \overline{c}_A\Ib_d~~~{\rm and}~~~\opnorm{\frac{1}{T}\sum_{t=1}^T \Ab_t - \E[\Ab_t]} \le \wt{O}\paren{C_a\sqrt{d/T} + d^{-4}},
\end{align*}
where $\underline{c}_A,\overline{c}_A,C_a>0$ are $\Theta(1)$ constants that depend only on $\gamma$. Therefore, taking $d=\wt{\Omega}(\max\set{\underline{c}_A^{-1/2}, 1})$ and $T\ge \wt{\Omega}(\max\set{4C_a^2/\underline{c}_A^2d, 4C_a^2/d})=\wt{\Omega}(d)$, we get that $\opnorm{\frac{1}{T}\sum_{t=1}^T \Ab_t - \E[\Ab_t]} \le \min\set{\underline{c}_A/2, 1/2}\le 1/2$ and $\lambda_{\min}(\frac{1}{T}\sum_{t=1}^T\Ab_t) \ge \underline{c}_A/2>0$. On the similar concentration event for $\Ab_t^2$ (in Lemma~\ref{lemma:concentration-atbt}), for $T\ge \wt{\Omega}(d)$ we also have $\opnorm{\frac{1}{T}\sum_{t=1}^T \Ab_t^2 - \E[\Ab_t^2]} \le 1/2$ and $\lambda_{\min}(\frac{1}{T}\sum_{t=1}^T \Ab_t^2) \ge \underline{c}_A/2>0$.

\paragraph{Concentration of MSE around expectation}
We first bound the concentration between the MSE and its expectation $R^2/d\cdot \tr(\bSigma_T)$. Define
\begin{align*}
  \zb = \begin{bmatrix}
    \wb_1 - \wb_{0,\star} \\
    \vdots \\
    \wb_T - \wb_{0,\star}
  \end{bmatrix} \in \R^{dT}~~~{\rm and}~~~\Ub = \begin{bmatrix}
    \paren{\sum_{t\le T} \Ab_t}^{-1} \Ab_1 \\
    \vdots \\
    \paren{\sum_{t\le T} \Ab_t}^{-1} \Ab_T
  \end{bmatrix} \in \R^{dT\times d}.
\end{align*}
Then we have $\hat{\wb}_{0,T}^{\nosp} - \wb_{0,\star}=\Ub^\top \zb$ and thus
$\norm{\hat{\wb}_{0,T}^{\nosp} - \wb_{0,\star}}^2 = \zb^\top(\Ub\Ub^\top)\zb$. By Assumption~\ref{assumption:realizable}, $\zb\in\R^{dT}$ has i.i.d. mean-zero $O(R^2/d)$-sub-Gaussian entries. Therefore, applying the Hanson-Wright inequality (Lemma~\ref{lemma:hanson-wright}) with $\Cb=\Ub\Ub^\top$, we get that with probability at least $1-\delta$ we have
\begin{equation}
  \label{equation:mse-concentration}
\begin{aligned}
  & \quad \abs{ \norm{\hat{\wb}_{0,T}^{\nosp} - \wb_{0,\star}}^2 - \frac{R^2}{d}\tr(\bSigma_T) } = \abs{ \zb^\top \Cb \zb - \E\brac{\zb^\top \Cb\zb} } \\
  & \le \wt{O}\paren{ \frac{R^2}{d} \lfro{\Cb}} = \wt{O}\paren{ \frac{R^2}{d}\lfro{\Ub^\top\Ub} }\\
  & = \wt{O}\paren{ \frac{R^2}{d} \lfro{ \paren{\sum_{t\le T} \Ab_t}^{-2} \sum_{t\le T}\Ab_t^2 } } \\
  & = \wt{O}\paren{ \frac{R^2}{dT} \lambda_{\min}\paren{\frac{1}{T}\sum_{t\le T} \Ab_t  }^{-2} \cdot \sqrt{d}\opnorm{\frac{1}{T}\sum_{t\le T}\Ab_t^2} } = \wt{O}\paren{ \frac{R^2}{T} \cdot \frac{1}{\sqrt{d}} }.
\end{aligned}
\end{equation}

\paragraph{Concentration of $\tr(\bSigma_T)$}
Recall that $\Ab_t$ are i.i.d. PSD matrices in $\R^{d\times d}$. We have
\begin{align*}
  & \quad \frac{R^2}{d} \cdot \tr(\bSigma_T) = \frac{R^2}{Td} \<\paren{ \frac{\sum_{t=1}^T \Ab_t}{T} }^{-2} , \frac{\sum_{t=1}^T \Ab_t^2}{T} \> \\
  & = \frac{R^2}{T} \Bigg\{ \underbrace{\frac{1}{d}\<\E[\Ab_1]^{-2}, \E[\Ab_1^2] \>}_{\rm I} + \underbrace{\frac{1}{d}\<\paren{\frac{\sum_{t=1}^T \Ab_t}{T} }^{-2} - \E[\Ab_1]^{-2}, \E[\Ab_1^2] \>}_{\rm II} + \\
  & \qquad \underbrace{\frac{1}{d}\<\paren{\frac{\sum_{t=1}^T \Ab_t}{T} }^{-2}, \frac{\sum_{t=1}^T \Ab_t^2}{T} - \E[\Ab_1^2]\>}_{\rm III} \Bigg\}.
\end{align*}
By Lemma~\ref{lemma:moments} and~\ref{lemma:moments-bounds}, term I is the main $\Theta(1)$ term:
\begin{align*}
  & \quad {\rm I} = \frac{1}{d}\< f_{A}(n,d,\lambda)^{-2}\Ib_d, f_{A^2}(n,d,\lambda)\Ib_d \> = f_{A^2}(n,d,\lambda) / f_A(n,d,\lambda)^2 \\
  & = \frac{\frac{1}{d}\E\brac{\tr\paren{\lambda^4(\hat{\bSigma}_n+\lambda\Ib_d)^{-4}\hat{\bSigma}_n^2}} }{ \paren{ \frac{1}{d}\E\brac{\tr\paren{\lambda^2(\hat{\bSigma}_n+\lambda\Ib_d)^{-2}\hat{\bSigma}_n}} }^2 }
    \eqdef C_{d,n,\lambda}^{\nosp} = \Theta(1).
\end{align*}
For term II we have
\begin{align*}
  & \quad \abs{\rm II} \le \opnorm{ \paren{\sum_{t=1}^T \Ab_t/T }^{-2} - \E[\Ab_1]^{-2} } \cdot \opnorm{\E[\Ab_1^2]} \\
  & \le \lambda_{\min}\paren{ \sum_{t=1}^T \Ab_t/T }^{-2} \opnorm{ \paren{\sum_{t=1}^T \Ab_t/T }^{2} - \E[\Ab_1]^{2} } \lambda_{\min}\paren{ \E[\Ab_1] }^{-2} \cdot \opnorm{ \E[\Ab_1^2] } \\
  & \le \wt{O}\paren{ \opnorm{ \paren{\sum_{t=1}^T \Ab_t/T }^{2} - \E[\Ab_1]^{2} } } \\
  &\le \wt{O}\paren{ \max\set{\opnorm{\sum_{t=1}^T \Ab_t/T}, \opnorm{ \E[\Ab_1] }} \cdot \opnorm{ \paren{\sum_{t=1}^T \Ab_t/T } - \E[\Ab_1] }} \le \wt{O}\paren{\sqrt{d/T} + d^{-4}}.
\end{align*}
Similarly we also have $\abs{\rm III}\le \wt{O}(\sqrt{d/T})$. Combining terms I, II, III, we get that (on the concentration event)
\begin{align*}
  \frac{R^2}{d}\tr\paren{\bSigma_T} = \frac{R^2}{T} \paren{C^{\nosp}_{d,n,\lambda} + \wt{O}\paren{\sqrt{d/T} + d^{-4}}}.
\end{align*}
This further combined with~\eqref{equation:mse-concentration} gives
\begin{align*}
  \norm{\hat{\wb}_{0,T}^{\nosp} - \wb_{0,\star}}^2 = \frac{R^2}{d}\tr\paren{\bSigma_T} + \frac{R^2}{T}\cdot \wt{O}(1/\sqrt{d}) = \frac{R^2}{T} \paren{C^{\nosp}_{d,n,\lambda} + \wt{O}\paren{\sqrt{d/T}} + \wt{O}\paren{1/\sqrt{d}}}.
\end{align*}
This proves the desired result for the \trtr method.

For the \trval method, observe that all the above analysis still holds if we replace $\Ab_t$ with $\Bb_t$ (and using the concentration for $\Bb_t$ guaranteed in Lemma~\ref{lemma:concentration-atbt}), we obtain a similar conclusion
\begin{align*}
  \norm{\hat{\wb}_{0,T}^{\sp} - \wb_{0,\star}}^2 = \frac{R^2}{T} \paren{C^{\sp}_{d,n_1,n_2,\lambda} + \wt{O}\paren{\sqrt{d/T}} + \wt{O}\paren{1/\sqrt{d}}},
\end{align*}
where
\begin{align*}
  & \quad C^{\sp}_{d,n_1,n_2,\lambda} \defeq f_{B^2}(d,n_1,n_2,\lambda) / f_B(d,n_1,n_2,\lambda)^2 \\
  & = \frac{   \frac{1}{dn_2}\E\brac{ \tr\paren{\lambda^2 (\hat{\bSigma}_{n_1} + \lambda\Ib_d)^{-2} }^2 + (n_2+1) \tr\paren{ \lambda^4(\hat{\bSigma}_{n_1} + \lambda\Ib_d)^{-4} } }   }{  \paren{  \frac{1}{d} \E\brac{ \tr\paren{ \lambda^2(\hat{\bSigma}_{n_1} + \lambda\Ib_d)^{-2} } }  }^2  }.
\end{align*}
This is the desired result.
\qed

\section{Proof of Theorem~\ref{theorem:comparison-mse}}
\label{appendix:proof-comparison-mse}

The proof is organized as follows. We optimize the hyperparameter $(\lambda, n_1)$ for the \trval method in Section~\ref{appendix:optimal-hyperparam-trval}. We derive the exact limit of the $C^{\nosp}_{d,n,\lambda}$ in the proportional limit $d,n\to\infty$, $d/n\to\gamma$ and optimize over $\lambda$ in Section~\ref{appendix:optimal-hyperparam-trtr}. We prove the main theorem in Section~\ref{appendix:proof-main-comparison-mse}.

\subsection{Optimizing the hyperparameters for the \trval method}
\label{appendix:optimal-hyperparam-trval}

\begin{lemma}[Optimal constant of the \trval method]
  \label{lemma:optimally-tuned-rates}
  For any $(n,d)$ and any split ratio $(n_1, n_2)=(n_1, n-n_1)$, the
  optimal constant (by tuning the regularization $\lambda>0$) of the \trval method is achieved at
  \begin{align*}
    \inf_{\lambda>0} C^{\sp}_{d, n_1, n_2, \lambda} =
    \lim_{\lambda\to\infty} C^{\sp}_{d, n_1, n_2, \lambda} = \frac{d+n_2+1}{n_2}.
  \end{align*}
  Further optimizing the rate over $n_2$,  the best rate is taken at $(n_1,n_2)=(0,n)$, in which the rate is
  \begin{align*}
    \inf_{\lambda>0,~n_2\in[n]} C^{\sp}_{d, n_1, n_2, \lambda} = \frac{(d+n+1)R^2}{n}.
  \end{align*}
\end{lemma}

\paragraph{Discussion: Using all data as validation}
Lemma~\ref{lemma:optimally-tuned-rates} suggests that the optimal constant of the \trval method is obtained at $\lambda=\infty$ and $(n_1,n_2)=(0,n)$. In other words, the optimal choice for the \trval method is to \emph{use all the data as validation}. In this case, since there is no training data, the inner solver reduces to the identity map: $\cA_{\infty, 0}(\wb_0;\Xb_t, \yb_t)=\wb_0$, and the outer loop reduces to learning a single linear model $\wb_0$ on all the tasks combined. We remark that while the optimality of such a split ratio is likely an artifact of the data distribution we assumed (noiseless realizable linear model) and may not generalize to other meta-learning problems, we do find experimentally that using more data as validation (than training) can also improve the performance on real meta-learning tasks (see Table~\ref{comparisontable2}).

\begin{proof-of-lemma}[\ref{lemma:optimally-tuned-rates}]
Fix $n_1\in[n]$ and $n_2=n-n_1$. Recall from
Theorem~\ref{theorem:concentration-mse} (with the eigenvalue-based expressions in Lemma~\ref{lemma:moments}) that
\begin{align*}
  C^{\sp}_{d,n_1,n_2,\lambda} = \frac{d}{n_2} \cdot \frac{\E\brac{ \paren{\sum_{i=1}^d \lambda^2/(\sigma_i^{(n_1)} + \lambda)^2}^2 + (n_2+1)\sum_{i=1}^d \lambda^4/(\sigma_i^{(n_1)} + \lambda)^4}}{\paren{ \E\brac{ \sum_{i=1}^d \lambda^2/(\sigma_i^{(n_1)} + \lambda)^2 } }^2}.
\end{align*}
Clearly, as $\lambda\to\infty$, we have
\begin{align*}
  \lim_{\lambda\to\infty} C^{\sp}_{d,n_1,n_2,\lambda}  = \frac{d}{n_2} \cdot \frac{d^2+(n_2+1)d}{d^2} = \frac{(d+n_2+1)}{n_2}.
\end{align*}
It remains to show that the above quantity is a lower bound for $C^{\sp}_{d,n_1,n_2,\lambda}$ for any $\lambda>0$, which is equivalent to
\begin{align}
  \label{equation:l2-ineq}
  \frac{\E\brac{ \paren{\sum_{i=1}^d \lambda^2/(\sigma_i^{(n_1)} + \lambda)^2}^2 + (n_2+1)\sum_{i=1}^d \lambda^4/(\sigma_i^{(n_1)} + \lambda)^4}}{\paren{ \E\brac{ \sum_{i=1}^d \lambda^2/(\sigma_i^{(n_1)} + \lambda)^2 } }^2} \ge \frac{d+n_2+1}{d},~~~\textrm{for all}~\lambda>0.
\end{align}
We now prove~\eqref{equation:l2-ineq}. For $i\in[n_1]$, define random variables
\begin{align*}
  X_i \defeq \frac{\lambda^2}{(\sigma_i^{(n_1)} + \lambda)^2}\in[0,1]~~~{\rm and}~~~Y_i \defeq 1 - X_i \in [0,1].
\end{align*}
Then the left-hand side of~\eqref{equation:l2-ineq} can be rewritten as
\begin{align*}
  & \quad \frac{\E\brac{ \paren{d - n_1 + \sum_{i=1}^{n_1}X_i}^2 + (n_2+1)\paren{ d - n_1 + \sum_{i=1}^{n_1}X_i^2 } }}{\paren{ \E\brac{ d - n_1 + \sum_{i=1}^n X_i } }^2} \\
  & = \frac{ \E\brac{ \paren{d - \sum_{i=1}^{n_1}Y_i}^2 + (n_2+1)\paren{d - 2\sum_{i=1}^{n_1}Y_i + \sum_{i=1}^{n_1} Y_i^2} }}{ \paren{\E\brac{ d - \sum_{i=1}^{n_1} Y_i }}^2 } \\
  & = \frac{d^2+(n_2+1)d - 2(d+n_2+1)\E\brac{\sum Y_i} + \E\brac{ (\sum Y_i)^2 } + (n_2+1)\E\brac{\sum Y_i^2}}{d^2 - 2d\E\brac{\sum Y_i} + \paren{\E\brac{ \sum Y_i }}^2}
\end{align*}
By algebraic manipulation, inequality~\eqref{equation:l2-ineq} is equivalent to showing that
\begin{align}
  \label{equation:l2-ineq-simplify}
  \frac{\E\brac{ (\sum Y_i)^2 } + (n_2+1)\E\brac{\sum Y_i^2}}{\paren{\E\brac{ \sum Y_i }}^2} \ge \frac{d+n_2+1}{d}.
\end{align}
Clearly, $\E[(\sum Y_i)^2] \ge (\E[\sum Y_i])^2$. By Cauchy-Schwarz we also have
\begin{align*}
  \E\brac{\sum Y_i^2} \ge \frac{1}{n_1} \E\brac{ \paren{\sum Y_i}^2 } \ge \frac{1}{n_1}\paren{ \E\brac{\sum Y_i} }^2.
\end{align*}
Therefore we have
\begin{align*}
  \frac{\E\brac{ (\sum Y_i)^2 } + (n_2+1)\E\brac{\sum Y_i^2}}{\paren{\E\brac{ \sum Y_i }}^2} \ge 1 + \frac{n_2+1}{n_1} \ge 1 + \frac{n_2+1}{d} = \frac{d+n_2+1}{d},
\end{align*}
where we have used that $n_1\le n\le d$. This shows~\eqref{equation:l2-ineq-simplify} and consequently~\eqref{equation:l2-ineq}.
\end{proof-of-lemma}

\subsection{Optimizing the hyperparameters for the  \trtr method (in the proportional limit)}
\label{appendix:optimal-hyperparam-trtr}

\begin{theorem}[Exact constant of the \trtr method in the proportional limit]
  \label{theorem:high-dim-limit}
  In the high-dimensional limiting regime $d,n\to\infty$, $d/n\to\gamma$  where $\gamma\in(0,\infty)$ is a fixed shape parameter,  for any $\lambda>0$  
  \begin{align*}
\lim\nolimits_{d,n\to\infty, d/n=\gamma} C^{\nosp}_{d,n,\lambda}
    = \rho_{\lambda,\gamma}  .
  \end{align*}
  where $\rho_{\lambda,\gamma} =  4\gamma^2 \brac{(\gamma-1)^2 + (\gamma+1)\lambda}/(\lambda +1+\gamma - \sqrt{(\lambda+\gamma+1)^2-4\gamma})^2 / \paren{(\lambda+\gamma+1)^2-4\gamma}^{3/2} $.
\end{theorem}

\begin{proof-of-theorem}[\ref{theorem:high-dim-limit}]
Let $\hat{\bSigma}_n \defeq \frac{1}{n} \Xb_t\Xb_t^\top$
denote the sample covariance matrix of the inputs in a single task
($t$). By Theorem~\ref{theorem:concentration-mse} (with the eigenvalue-based expressions in Lemma~\ref{lemma:moments}), we have
\begin{equation}
  \label{equation:l1-rate-rewrite}
  \begin{aligned}
    & \quad C^{\nosp}_{d,n,\lambda}
    = \frac{\frac{1}{d}\E\brac{ \sum_{i=1}^d \sigma_i(\hat{\bSigma}_n)^2 / (\sigma_i(\hat{\bSigma}_n) + \lambda)^4}}{ \paren{ \frac{1}{d}\E\brac{ \sum_{i=1}^d \sigma_i(\hat{\bSigma}_n) / (\sigma_i(\hat{\bSigma}_n) + \lambda)^2 } }^2 } \\
    & = \underbrace{\frac{1}{d}\E\brac{ \tr\paren{(\hat{\bSigma}_n + \lambda \Ib_d)^{-4}\hat{\bSigma}_n^2} }}_{{\rm I}_{n,d}} \Big/ \bigg\{ \underbrace{\frac{1}{d}\E\brac{ \tr\paren{ (\hat{\bSigma}_n + \lambda \Ib_d)^{-2}\hat{\bSigma}_n } }}_{{\rm II}_{n,d}}\bigg\}^2.
  \end{aligned}
\end{equation}
We now evaluate quantities ${\rm I}_{n,d}$ and ${\rm II}_{n,d}$ in the high-dimensional limit of $d,n\to\infty$, $d/n\to\gamma\in(0,\infty)$. Consider the (slightly generalized) Stieltjes transform of $\hat{\bSigma}_n$ defined for all $\lambda_1, \lambda_2>0$:
\begin{align}
  \label{equation:stj-def}
  s(\lambda_1, \lambda_2) \defeq \lim_{d, n\to\infty,~d/n\to\gamma} \frac{1}{d}\E\brac{ \tr\paren{ (\lambda_1 \Ib_d + \lambda_2 \hat{\bSigma}_n)^{-1} } }.
\end{align}
As the entries of $\Xb_t$ are i.i.d. $\normal(0, 1)$, the above limiting Stieltjes transform is the Stieltjes form of the Marchenko-Pastur law, which has a closed form (see, e.g.~\citep[Equation (7)]{dobriban2018high})
\begin{equation}
  \label{equation:stj}
  \begin{aligned}
    & \quad s(\lambda_1, \lambda_2) = \lambda_2^{-1} s(\lambda_1/\lambda_2, 1)
    = \frac{1}{\lambda_2} \cdot \frac{\gamma - 1 - \lambda_1/\lambda_2 + \sqrt{(\lambda_1/\lambda_2 + 1 + \gamma)^2 - 4\gamma}}{2\gamma\lambda_1/\lambda_2} \\
    & = \frac{ \gamma - 1 - \lambda_1/\lambda_2 + \sqrt{(\lambda_1/\lambda_2 + 1 + \gamma)^2 - 4\gamma} }{2\gamma\lambda_1}.
  \end{aligned}
\end{equation}
Now observe that differentiating $s(\lambda_1, \lambda_2)$ yields quantity II (known as the derivative trick of Stieltjes transforms). Indeed, we have
\begin{equation}
  \label{equation:exchange-derivative}
\begin{aligned}
  & \quad -\frac{d}{d\lambda_2} s(\lambda_1, \lambda_2) = -\frac{d}{d\lambda_2} \lim_{d, n\to\infty,~d/n\to\gamma} \frac{1}{d}\E\brac{ \tr\paren{ (\lambda_1 \Ib_d + \lambda_2 \hat{\bSigma}_n)^{-1} } } \\
  & = \lim_{d, n\to\infty,~d/n\to\gamma} \frac{1}{d}\E\brac{ -\frac{d}{d\lambda_2} \tr\paren{ (\lambda_1 \Ib_d + \lambda_2 \hat{\bSigma}_n)^{-1}}} \\
  & = \lim_{d, n\to\infty,~d/n\to\gamma} \frac{1}{d}\E\brac{ \tr\paren{ (\lambda_1 \Ib_d + \lambda_2 \hat{\bSigma}_n)^{-2}\hat{\bSigma}_n}}.
\end{aligned}
\end{equation}
(Above, the exchange of differentiation and limit is due to the uniform convergence of the derivatives, which holds at any $\lambda_1,\lambda_2>0$. See Section~\ref{appendix:exchange} for a detailed justification.) Taking $\lambda_1=\lambda$ and $\lambda_2=1$, we get
\begin{align*}
  \lim_{d,n\to\infty,~d/n\to\gamma} {\rm II}_{n,d} = \lim_{d, n\to\infty,~d/n\to\gamma} \frac{1}{d}\E\brac{ \tr\paren{ (\lambda \Ib_d + \hat{\bSigma}_n)^{-2}\hat{\bSigma}_n}} = -\frac{d}{d\lambda_2} s(\lambda_1, \lambda_2)|_{\lambda_1=\lambda,\lambda_2=1}.
\end{align*}
Similarly we have
\begin{align*}
  \lim_{d,n\to\infty,~d/n\to\gamma} {\rm I}_{n,d} = \lim_{d,n\to\infty, d/n\to\gamma} \frac{1}{d}\E\brac{ \tr\paren{(\lambda \Ib_d + \hat{\bSigma}_n)^{-4}\hat{\bSigma}_n^2} } = -\frac{1}{6} \frac{d}{d\lambda_1}\frac{d^2}{d\lambda_2^2} s(\lambda_1, \lambda_2)|_{\lambda_1=\lambda,\lambda_2=1}.
\end{align*}
Evaluating the right-hand sides from differentiating the closed-form expression~\eqref{equation:stj}, we get
\begin{align*}
  & \lim_{d,n\to\infty,~d/n\to\gamma} {\rm II}_{n, d} = \frac{1}{2\gamma}\cdot \frac{\lambda+1+\gamma}{\sqrt{(\lambda+1+\gamma)^2 - 4\gamma}} - \frac{1}{2\gamma}, \\
  & \lim_{d,n\to\infty,~d/n\to\gamma} {\rm I}_{n, d} = \frac{(\gamma-1)^2 + (\gamma+1)\lambda}{\paren{(\lambda+1+\gamma)^2 - 4\gamma}^{5/2}}.
\end{align*}
Substituting back to~\eqref{equation:l1-rate-rewrite} yields that
\begin{align*}
  & \quad \lim_{d,n\to\infty,~d/n\to\gamma} C^{\nosp}_{d,n,\lambda} = \lim_{d,n\to\infty,~d/n\to\gamma} \cdot {\rm I}_{n,d} / {\rm II}_{n,d}^2 \\
  & = \frac{4\gamma^2\brac{(\gamma-1)^2 + (\gamma+1)\lambda}}{((\lambda+1+\gamma)^2 - 4\gamma)^{5/2} \cdot \paren{ \frac{\lambda+1+\gamma}{\sqrt{(\lambda+1+\gamma)^2 - 4\gamma}} - 1 }^2} \\
  & = \frac{4\gamma^2\brac{(\gamma-1)^2 + (\gamma+1)\lambda}}{((\lambda+1+\gamma)^2 - 4\gamma)^{3/2} \cdot \paren{ \lambda+1+\gamma - \sqrt{(\lambda+1+\gamma)^2 - 4\gamma}}^2}. 
\end{align*}
This proves the desired result.
\end{proof-of-theorem}

\subsubsection{Exchanging derivative and expectation / limit}
\label{appendix:exchange}
Here we rigorously establish the exchange of the derivative and the expectation / limit used in~\eqref{equation:exchange-derivative}. For convenience of notation let $\bSigma=\hat{\bSigma}_n=\Xb_t^\top\Xb_t/n$ denote the empirical covariance matrix of $\Xb_t$. We wish to show that
\begin{align*}
  \frac{d}{d\lambda_2} \lim_{d,n\to\infty,d/n\to\gamma} \frac{1}{d}\E\brac{\tr\paren{ (\lambda_1\Ib_d + \lambda_2 \bSigma)^{-1} }} = \lim_{d,n\to\infty,d/n\to\gamma} \frac{1}{d}\E\brac{ \frac{d}{d\lambda_2} \tr\paren{ (\lambda_1\Ib_d + \lambda_2 \bSigma)^{-1} }}.
\end{align*}
This involves the exchange of derivative and limit, and then the exchange of derivative and expectation. 

\paragraph{Exchange of derivative and expectation}
First, we show that for any fixed $(d,n)$,
\begin{align*}
  \frac{d}{d\lambda_2} \E\brac{ \tr\paren{(\lambda_1 \Ib_d + \lambda_2 \bSigma)^{-1}}} = \E\brac{\frac{d}{d\lambda_2} \tr\paren{(\lambda_1 \Ib_d + \lambda_2 \bSigma)^{-1}} }.
\end{align*}
By definition of the derivative, we have
\begin{align*}
  \frac{d}{d\lambda_2}\E\brac{ \tr\paren{(\lambda_1 \Ib_d + \lambda_2 \bSigma)^{-1}}}  = \lim_{t\to 0} \E\brac{ \frac{\tr\paren{(\lambda_1 \Ib_d + \lambda_2 \bSigma + t\bSigma)^{-1}} - \tr\paren{(\lambda_1 \Ib_d + \lambda_2 \bSigma)^{-1}}}{t} }.
\end{align*}
For any $\Ab\succ \bzero$, the function $t\mapsto \tr((\Ab+t\Bb)^{-1})$ is continuously differentiable at $t=0$ with derivative $-\tr(\Ab^{-2}\Bb)$, and thus locally Lipschitz around $t=0$ with Lipschitz constant $|\tr(\Ab^{-2}\Bb)|+1$. Applying this in the above expectation with $\Ab=\lambda_1\Ib_d+\lambda_2\bSigma\succeq \lambda_1\Ib_d$ and $\Bb=\bSigma$, we get that for sufficiently small $|t|$, the fraction inside the expectation is upper bounded by $|\tr(\lambda_1^{-2}\bSigma)|+1<\infty$ uniformly over $t$. Thus by the Dominated Convergence Theorem, the limit can be passed into the expectation, which yields the expectation of the derivative.

\paragraph{Exchange of derivative and limit}
Define $f_{n,d}(\lambda_2)\defeq \frac{1}{d}\E\brac{\tr\paren{(\lambda_1\Ib_d+\lambda_2\bSigma)^{-1}}}$. It suffices to show that
\begin{align*}
  \frac{d}{d\lambda_2} \lim_{d,n\to\infty,d/n\to\gamma} f_{n,d}(\lambda_2) = \lim_{d,n\to\infty,d/n\to\gamma} f_{n,d}'(\lambda_2),
\end{align*}
where
\begin{align*}
  f_{n,d}'(\lambda_2) = \E\brac{\frac{d}{d\lambda_2} \frac{1}{d} \tr\paren{(\lambda_1 \Ib_d + \lambda_2 \bSigma)^{-1}} } = -\frac{1}{d}\E\brac{\tr\paren{(\lambda_1 \Ib_d + \lambda_2 \bSigma)^{-2}\bSigma} }
\end{align*}
by the result of the preceding part.

As $f_{n,d}(\lambda_2)\to s(\lambda_1,\lambda_2)$ pointwise over $\lambda_2$ by properties of the Wishart matrix~\citep{bai2010spectral} and each individual $f_{n,d}$ is differentiable, it suffices to show that the derivatives $f_{n,d}'(\wt{\lambda}_2)$ converges uniformly for $\wt{\lambda}_2$ in a neighborhood of $\lambda_2$. Observe that can rewrite $f_{n,d}'$ as
\begin{align*}
  f_{n,d}'(\wt{\lambda}_2) = - \E_{\hat{\mu}_{n,d}}\brac{ \E_{\lambda\sim\hat{\mu}_{n,d}}\brac{g_{\wt{\lambda}_2}(\lambda)} },
\end{align*}
where $\hat{\mu}_{n,d}$ is the empirical distribution of the eigenvalues of $\bSigma$, and
\begin{align*}
  g_{\wt{\lambda}_2}(\lambda) \defeq \frac{\lambda}{(\lambda_1+\wt{\lambda}_2\lambda)^2} \le \frac{1}{\lambda_1\wt{\lambda}_2}~~~\textrm{for all}~\lambda\ge 0.
\end{align*}
Therefore, as $\hat{\mu}_{n,d}$ converges weakly to the Marchenko-Pastur distribution with probability one and $g_{\wt{\lambda}_2}$ is uniformly bounded for $\wt{\lambda}_2$ in a small neighborhood of $\lambda_2$, we get that $f_{n,d}'(\wt{\lambda}_2)$ does converge uniformly to the expectation of $g_{\wt{\lambda}_2}(\lambda)$ under the Marchenko-Pastur distribution. This shows the desired exchange of derivative and limit.

\subsection{Proof of main theorem}
\label{appendix:proof-main-comparison-mse}

We are now ready to prove the main theorem (Theorem~\ref{theorem:comparison-mse}).

\paragraph{Part I: Optimal constant for $L^{\nosp}$}
By Theorem~\ref{theorem:high-dim-limit}, we have
\begin{align*}
  & \quad  \inf_{\lambda>0} \lim_{d,n\to\infty, d/n=\gamma} C^{\nosp}_{d,n,\lambda} \\
    & = \inf_{\lambda>0} \underbrace{\frac{4\gamma^2 \brac{(\gamma-1)^2 + (\gamma+1)\lambda}}{(\lambda+1+\gamma - \sqrt{(\lambda+\gamma+1)^2-4\gamma})^2 \cdot \paren{(\lambda+\gamma+1)^2-4\gamma}^{3/2}}}_{\defeq f(\lambda, \gamma)}.
\end{align*}
In order to bound $\inf_{\lambda>0}f(\lambda, \gamma)$, picking any $\lambda=\lambda(\gamma)$ gives $f(\lambda(\gamma), \gamma)$ as a valid upper bound, and our goal is to choose $\lambda$ that yields a bound as tight as possible. Here we consider the choice
\begin{align*}
  \lambda = \lambda(\gamma) = \max\set{1 - \gamma/2, \gamma-1/2} = (1-\gamma/2)\indic{\gamma \le 1} + (\gamma-1/2)\indic{\gamma > 1}
\end{align*}
which we now show yields the claimed upper bound.

\paragraph{Case 1: $\gamma \le 1$} Substituting $\lambda=1-\gamma/2$ into $f(\lambda, \gamma)$ and simplifying, we get
\begin{align*}
  f(1-\gamma/2, \gamma) = \frac{2(\gamma^2-3\gamma+4)}{(2-\gamma/2)^3} \eqdef g_1(\gamma).
\end{align*}
Clearly, $g_1(0)=1$ and $g_1(1) = 32/27$. Further differentiating $g_1$ twice gives
\begin{align*}
  g_1''(\gamma) = \frac{\gamma^2+7\gamma+4}{(2-\gamma/2)^5} > 0~~~\textrm{for all}~\gamma\in[0, 1]. 
\end{align*}
Thus $g_1$ is convex on $[0,1]$, from which we conclude that
\begin{align*}
  g_1(\gamma) \le (1-\gamma)\cdot g_1(0) + \gamma \cdot g_1(1) = 1+\frac{5}{27}\gamma. 
\end{align*}

\paragraph{Case 2: $\gamma > 1$} Substituting $\lambda=\gamma-1/2$ into $f(\lambda,\gamma)$ and simplifying, we get
\begin{align*}
  f(\gamma-1/2, \gamma) = \frac{2\gamma^2(4\gamma^2-3\gamma+1)}{(2\gamma-1/2)^3} \eqdef g_2(\gamma).
\end{align*}
We have $g_2(1)=g_1(1)=32/27$. Further differentiating $g_2$ gives
\begin{align*}
  g_2'(\gamma) = -\frac{1}{(4\gamma-1)^2} - \frac{6}{(4\gamma-1)^3} - \frac{6}{(4\gamma-1)^4} + 1 < 1~~~\textrm{for all}~\gamma > 1.
\end{align*}
Therefore we have for all $\gamma>1$ that
\begin{align*}
  g_2(\gamma) = g_2(1) + \int_1^\gamma g_2'(t)dt \le g_2(1) + \gamma - 1 = \gamma + \frac{5}{27}.
\end{align*}

Combining Case 1 and 2, we get
\begin{align*}
  & \quad \inf_{\lambda>0} f(\lambda, \gamma) \le g_1(\gamma) \\
  & \le \indic{\gamma\le 1} + g_2(\gamma) \indic{\gamma > 1} \le \paren{1+\frac{5}{27}\gamma}\indic{\gamma\le 1} + \paren{\frac{5}{27} + \gamma}\indic{\gamma > 1} \\
  & = \max\set{ 1+\frac{5}{27}\gamma, \frac{5}{27} + \gamma}.
\end{align*}
This is the desired upper bound for $L^{\nosp}$.

\paragraph{Equality at $\gamma=1$}
We finally show that the above upper bound becomes an equality when $\gamma=1$. At $\gamma=1$, we have
\begin{align*}
  f(\lambda, 1) = \frac{8\lambda}{(\lambda+2-\sqrt{\lambda^2+4\lambda})^2 (\lambda^2+4\lambda)^{3/2}} = \frac{8\lambda^{-4}}{(1 + 2/\lambda - \sqrt{1 + 4/\lambda})^2 (1+4/\lambda)^{3/2}}.
\end{align*}
Make the change of variable $t=\sqrt{1+4/\lambda}$ so that $\lambda^{-1} = (t^2-1)/4$, minimizing the above expression is equivalent to minimizing
\begin{align*}
  \frac{(t^2-1)^4/32}{(t^2/2-t+1/2)^2t^3} = \frac{(t+1)^4}{8t^3}
\end{align*}
over $t>1$. It is straightforward to check (by computing the first and second derivatives) that the above quantity is minimized at $t=3$ with value $32/27$. In other words, we have shown
\begin{align*}
  \inf_{\lambda>0} f(\lambda, 1) = \frac{32}{27} = \max\set{1 + \frac{5}{27}\gamma, \frac{5}{27}+\gamma}\bigg|_{\gamma=1},
\end{align*}
that is, the equality holds at $\gamma=1$. 

\paragraph{Part II: Optimal constant for $L^{\sp}$}
We now prove the result on $L^{\sp}$, that is,
\begin{align*}
  & \quad \inf_{\lambda>0, s\in(0,1)} \lim_{d,n\to\infty, d/n=\gamma} C^{\sp}_{d,ns,n(1-s),\lambda} \\
  & \stackrel{(i)}{=} \lim_{d,n\to\infty, d/n=\gamma} \underbrace{\inf_{\lambda>0, n_1+n_2=n} C^{\sp}_{d,n_1,n_2,\lambda}  }_{\frac{d+n+1}{n}}
    \stackrel{(ii)}{=} 1+\gamma.
\end{align*}
First, equality (ii) follows from Lemma~\ref{lemma:optimally-tuned-rates} and the fact that $(d+n+1)/n\to 1+\gamma$. Second, the ``$\ge$'' direction of equality (i) is trivial (since we always have ``$\inf \lim \ge \lim\inf$''). Therefore we get the ``$\ge$'' direction of the overall equality, and it remains to prove the ``$\le$'' direction.

For the ``$\le$'' direction, we fix any $\lambda>0$, and bound $C^{\sp}_{d,n_1,n_2,\lambda}$ (and consequently its limit as $d,n\to\infty$.) We have by Theorem~\ref{theorem:concentration-mse} (with the eigenvalue-based expressions in Lemma~\ref{lemma:moments}) that
\begin{align*}
  & \quad C^{\sp}_{d,n_1,n_2,\lambda} = \frac{d}{n_2} \cdot \frac{\E\brac{ \paren{\sum_{i=1}^d \lambda^2/(\sigma_i^{(n_1)} + \lambda)^2}^2 + (n_2+1)\sum_{i=1}^d \lambda^4/(\sigma_i^{(n_1)} + \lambda)^4}}{\paren{ \E\brac{ \sum_{i=1}^d \lambda^2/(\sigma_i^{(n_1)} + \lambda)^2 } }^2} \\
  & \le \frac{d}{n_2} \cdot \frac{d^2+(n_2+1)d}{ \paren{ \E\brac{ \sum_{i=1}^d \lambda^2/(\sigma_i^{(n_1)} + \lambda)^2 } }^2 } \\
  & = \frac{d+n_2+1}{n_2} \cdot \frac{1}{ \paren{ \E\brac{ \frac{1}{d}\sum_{i=1}^d \lambda^2 / (\sigma_i^{(n_1)} +\lambda)^2 } }^2 } \\
\end{align*}
Observe that
\begin{align*}
  & \quad \E\brac{ \frac{1}{d}\sum_{i=1}^d \frac{\lambda^2}{ (\sigma_i^{(n_1)} +\lambda)^2} }
    \stackrel{(i)}{\ge} \E\brac{ \frac{\lambda^2}{ \paren{ \sum_{i=1}^d \sigma_i^{(n_1)}/d + \lambda }^2} } \\
  & \stackrel{(ii)}{\ge} \frac{\lambda^2}{ \paren{\E\brac{\sum_{i=1}^d \sigma_i^{(n_1)}/d} + \lambda}^2}
    \stackrel{(iii)}{=} \frac {\lambda^2}{(1+\lambda)^2},
\end{align*}
where (i) follows from the convexity of $t\mapsto \lambda^2/(t+\lambda)^2$ on $t\ge 0$; (ii) follows from the same convexity and Jensen's inequality, and (iii) is since $\E\brac{\sum_{i=1}^d \sigma_i^{(n_1)}} = \E\brac{\tr(\frac{1}{n_1}\Xb_t^\top\Xb_t)} = \E\brac{\lfro{\Xb_t}^2/n_1} = d$. Applying this in the preceeding bound yields
\begin{align*}
  C^{\sp}_{d,n_1,n_2,\lambda} \le \frac{d+n_2+1}{n_2} \cdot \frac{(1+\lambda)^2}{\lambda^2}.
\end{align*}
Further plugging in $n_1=ns$ and $n_2=n(1-s)$ for any $s\in(0,1)$ yields
\begin{align*}
  \lim_{d,n\to\infty,d/n\to\gamma}  C^{\sp}_{d,ns,n(1-s),\lambda}  \le \frac{\gamma+1-s}{1-s} \cdot \frac{(1+\lambda)^2}{\lambda^2}.
\end{align*}
Finally, the right-hand side is minimized at $\lambda\to\infty$ and $s=0$, from which we conclude that
\begin{align*}
  \inf_{\lambda>0,s\in(0,1)} \lim_{d,n\to\infty,~d/n\to\gamma}  C^{\sp}_{d,ns,n(1-s),\lambda} \le 1+\gamma,
\end{align*}
which is the desired ``$\le$'' direction.
\qed

\section{Connections to Bayesian estimator}
\label{appendix:bayesian}
Here we discuss the relationship between our \trtr meta-learining estimator using ridge regression solvers and a Bayesian estimator under a somewhat natural hierarchical generative model for the realizable setting in Section~\ref{section:realizable}. We show that these two estimators are not equal in general, albeit they have some similarities in their expressions.

We consider the following hierarchical probabilitistic model:
\begin{align*}
  \wb_{0,\star} \sim \normal\paren{0, \frac{\sigma_w^2}{d}\Ib_d},~~~\wb_t | \wb_{0,\star}\simiid \normal\paren{\wb_{0,\star}, \frac{R^2}{d}\Ib_d},~~~\yb_t = \Xb_t\wb_t + \sigma \zb_t~~{\rm where}~\zb_t\simiid \normal(0, \Ib_n).
\end{align*}
This model is similar to our realizable linear model~\eqref{equation:realizable-model}, except that $\wb_0$ has a prior and that there is observation noise in the data (such that data likelihoods and posteriors are well-defined). We also note that the Gaussian distribution assumption (with $R^2/d$ variance) for $\wb_t$ is consistent with our Assumption~\ref{assumption:realizable}.

\paragraph{Bayesian estimator}
We now derive the Bayesian posterior mean estimator of $\wb_{0,\star}$, which requires us to compute the posterior distribution of $\wb_{0,\star}$ given the data $\set{(\Xb_t, \yb_t)}_{t=1}^T$\footnote{Hereafter we treat $\Xb_t$ as fixed, as the density of $\Xb_t$ won't affect the Bayesian calculation.}.

We begin by computing the likelihood of one task by marginalizing over $\wb_t$:
\begin{align*}
  & \quad p(\Xb_t, \yb_t | \wb_{0,\star}) \propto \int p(\wb_t | \wb_{0,\star}) \cdot p(\yb_t | \Xb_t, \wb_t) d\wb_t \\
  & \propto \int \exp\paren{-\frac{\norm{\wb_t - \wb_{0,\star}}^2}{2R^2/d}} \cdot \exp\paren{ -\frac{\norm{\yb_t - \Xb_t\wb_t}^2}{2\sigma^2} } d\wb_t \\
  & \stackrel{(i)}{\propto} \exp\paren{-\frac{\norm{\wb_{0,\star}}^2}{2R^2/d} + \frac{1}{2}\paren{\frac{\wb_{0,\star}}{R^2/d} + \frac{\Xb_t^\top\yb_t}{\sigma^2}}^\top \paren{ \frac{\Xb_t^\top\Xb_t}{\sigma^2} + \frac{\Ib_d}{R^2/d} }^{-1} \paren{\frac{\wb_{0,\star}}{R^2/d} + \frac{\Xb_t^\top\yb_t}{\sigma^2}}  } \\
  & \propto \exp\paren{-\frac{1}{2}\wb_{0,\star}^\top\paren{ \paren{\Xb_t^\top\Xb_t + \frac{d\sigma^2}{R^2}\Ib_d}^{-1}\frac{\Xb_t^\top\Xb_t}{R^2/d} }\wb_{0,\star} + \wb_{0,\star}^\top \paren{\Xb_t^\top\Xb_t + \frac{d\sigma^2}{R^2}\Ib_d}^{-1}\frac{\Xb_t^\top\yb_t}{R^2/d}},
\end{align*}
where (i) is obtained by integrating a multivariate Gaussian density over $\wb_t$, and ``$\propto$'' drops all the terms that do not depend on $\wb_{0,\star}$. Therefore, by the Bayes rule, the overall posterior distribution of $\wb_{0,\star}$ is given by
\begin{align*}
  & \quad p\paren{\wb_{0,\star} | \set{(\Xb_t, \yb_t)}_{t=1}^T} \propto p(\wb_{0,\star}) \cdot \prod_{t=1}^T p(\Xb_t, \yb_t | \wb_{0,\star}) \\
  & \propto \exp\paren{-\frac{\norm{\wb_{0,\star}}^2}{2\sigma_w^2/d}} \cdot \\
  & \qquad \prod_{t=1}^T \exp\paren{-\frac{1}{2}\wb_{0,\star}^\top\paren{ \paren{\Xb_t^\top\Xb_t + \frac{d\sigma^2}{R^2}\Ib_d}^{-1}\frac{\Xb_t^\top\Xb_t}{R^2/d} }\wb_{0,\star} + \wb_{0,\star}^\top \paren{\Xb_t^\top\Xb_t + \frac{d\sigma^2}{R^2}\Ib_d}^{-1}\frac{\Xb_t^\top\yb_t}{R^2/d}}.
\end{align*}
This means that the posterior distribution of $\wb_{0,\star}$ is Gaussian, with mean , i.e. the Bayesian estimator, equal to\footnote{Any density $p(\wb)\propto \exp(-\wb^\top \Ab\wb/2 + \wb^\top\cbb)$ specifies a Gaussian distreibution $\normal(\bmu, \bSigma)$, where $\Ab=\bSigma^{-1}$ and $\cbb=\bSigma^{-1}\bmu$, so that $\bmu=\Ab^{-1}\cbb$.}
\begin{align*}
  \hat{\wb}_{0,T}^{\sf Bayes}\defeq \E\brac{\wb_{0,\star} \mid \set{(\Xb_t, \yb_t)}_{t=1}^T} = (\Ab_T^{\sf Bayes})^{-1} \cbb_T^{\sf Bayes},
\end{align*}
where
\begin{align*}
  & \Ab_T^{\sf Bayes} \defeq \frac{d}{\sigma_w^2} \Ib_d + \sum_{t=1}^T \paren{\Xb_t^\top\Xb_t + \frac{d\sigma^2}{R^2}\Ib_d}^{-1}\frac{\Xb_t^\top\Xb_t}{R^2/d}, \\
  & \cbb_T^{\sf Bayes} \defeq \sum_{t=1}^T \paren{\Xb_t^\top\Xb_t + \frac{d\sigma^2}{R^2}\Ib_d}^{-1}\frac{\Xb_t^\top\yb_t}{R^2/d}.
\end{align*}
We note that $\hat{\wb}_{0,T}^{\sf Bayes}$ has a similar form as our \trtr estimator, but is not exactly the same. Indeed, recall the closed form of our \trtr estimator is (cf.~\eqref{equation:l1-global-min})
\begin{align*}
  \hat{\wb}^{\nosp}_{0,T} = (\Ab_T^{\nosp})^{-1}\cbb_T^{\nosp},
\end{align*}
where
\begin{align*}
  & \Ab_T^{\nosp} = \sum_{t=1}^T \paren{\Xb_t^\top\Xb_t + n\lambda \Ib_d}^{-2} \Xb_t^\top\Xb_t, \\
  & \cbb_T^{\nosp} = \sum_{t=1}^T \paren{\Xb_t^\top\Xb_t + n\lambda \Ib_d}^{-2} \Xb_t^\top\yb_t.
\end{align*}
As $\hat{\wb}^{\sf Bayes}_{0,T}$ uses the {\bf inverse} and $\hat{\wb}^{\nosp}_{0,T}$ uses the {\bf squared inverse}, these two sets of estimators are not the same in general, no matter how we tune the $\lambda$ in the \trtr estimator. This is true even if we set $\sigma_w=\infty$ so that the prior of $\wb_{0,\star}$ becomes degenerate (and the Bayesian estimator reduces to the MLE).

\section{Details on the few-shot image classification experiment}
\label{appendix:deep}
Here we provide additional details of the few-shot image classification experiment in Section~\ref{section:deep}.

\paragraph{Optimization and architecture}
For both methods, we run a few gradient steps on the inner argmin problem to obtain (an approximation of) $\wb_t$, and plug $\wb_t$ into the $\grad_{\wb_0} \ell_t^{\set{\sp,\nosp}}(\wb_0)$ (which involves $\wb_t$ through implicit function differentiation) for optimizing $\wb_0$ in the outer loop. 

For both \trtr and \trval methods, we use the standard $4$-layer convolutional network in~\citep{finn2017model,zhou2019efficient} as the backbone (i.e. the architecture for $\wb_t$). We further tune their hyper-parameters, such as the regularization constant $\lambda$, the learning rate (initial learning rate and its decay strategy), and the gradient clipping threshold. %

\paragraph{Dataset and evaluation}
MiniImageNet consists of $100$ classes of images from ImageNet~\citep{krizhevsky2012imagenet} and each class has $600$ images of resolution $84\times84\times 3$. We use $64$ classes for training, $16$ classes for validation, and the remaining $20$ classes for testing~\citep{ravi2016optimization}. TieredImageNet consists of $608$ classes from the ILSVRC-$12$ data set~\citep{russakovsky2015imagenet} and each image is also of resolution $84\times 84\times 3$. TieredImageNet groups classes into broader hierarchy categories corresponding to higher-level nodes in the ImageNet. Specifically, its top hierarchy has $20$ training categories ($351$ classes), $6$ validation categories ($97$ classes) and $8$ test categories (160 classes). This structure ensures that all training classes are distinct from the testing classes, providing a more realistic few-shot learning scenario.

We evaluate both methods under the transduction setting where the information is shared between the test data via batch normalization.

\subsection{Effect of the split ratio for the \trval method}
\label{appendix:split-ratio}
We further tune the split $(n_1, n_2)$ in the \trval method and report the results in Table~\ref{comparisontable2}. As can be seen, as the number of test samples $n_2$ increases, the percent classification accuracy on both the miniImageNet and tieredImageNet datasets becomes higher. This testifies our theoretical affirmation in Lemma~\ref{lemma:optimally-tuned-rates}. However, note that even if we take the best split  $(n_1,n_2)=(5,25)$ (and compare again with Table~\ref{comparisontable}), the \trval method still performs worse than the \trtr method.

We remark that our theoretical results on \trtr performing better than \trval (in Section~\ref{section:realizable}) rely on the assumptions that the data can be exactly realized by the representation and contains no label noise. Our experimental results here may suggest that the miniImageNet and tieredImageNet few-shot tasks may have a similar structure (there exists a NN representation that almost perfectly realizes the label with no noise) that allows the \trtr method to perform better than the \trval method.

\begin{table*}[h]
  \centering
  \caption{Investigation of the effects of  training/validation splitting ratio in  the \trval method (iMAML) to the few-shot classification accuracy ($\%$)  on  miniImageNet and tieredImageNet.}
  \label{comparisontable2}
  \begin{tabular}{c|ccc}
    \bottomrule 
    datasets  & $n_1=25, n_2=5$ &  $n_1=15, n_2=15$  &  $n_1=5, n_2=25$   \\
    \midrule
    miniImageNet &   62.09 $\pm$ 0.97   &	63.56 $\pm$ 0.95 &  63.92 $\pm$ 1.04   \\
    tieredImageNet &  66.45 $\pm$ 1.05   &67.30	 $\pm$ 0.98  &  67.50 $\pm$ 0.94    \\
    \bottomrule
  \end{tabular}
\end{table*}

\section{Comparison with Cross-Validation on Synthetic Data}
\label{appendix:cross-validation}

We test the effect of using cross-validation for the \trval method on the same synthetic data (realizable linear centroid meta-learning) as in Section \ref{sec:synthetic}.

\paragraph{Method}
We fix the number of per-task data $n = 20$, and use $4$-fold cross validation in the following two settings: $(n_1,n_2)=(5,15)$, and $(n_1,n_2)=(15,5)$. In both cases, we partition the data into 4 parts each with 5 data points, and we roulette over $4$ possible partitions of which one as train and which one as validation. The estimated optimal $\hat{\wb}_0^{\textrm{cv}}$ is obtained by minimize the averaged \trval loss over the 4 partitions:
\begin{align*}
\ell^{\textrm{cv}}_t(\wb_0) & \defeq \frac{1}{4} \sum_{j=1}^4 \frac{1}{2n_{\textrm{\val}}} \norm{ \yb_t^{\val, j} - \Xb_t^{\val, j} \cA_\lambda(\wb_0; \Xb_t^{\train, j}, \yb_t^{\train, j})}^2, \\
\hat{\wb}_0^{\textrm{cv}} & = \argmin_{\wb_0} \frac{1}{T}\sum_{t=1}^T \ell^{\textrm{cv}}_t(\wb_0),
\end{align*}
where superscript $j$ denotes the index of the cross-validation. The performance is depicted in Figure \ref{fig:cv}.
\begin{figure}[h]
\centering
\includegraphics[width = 0.5\textwidth]{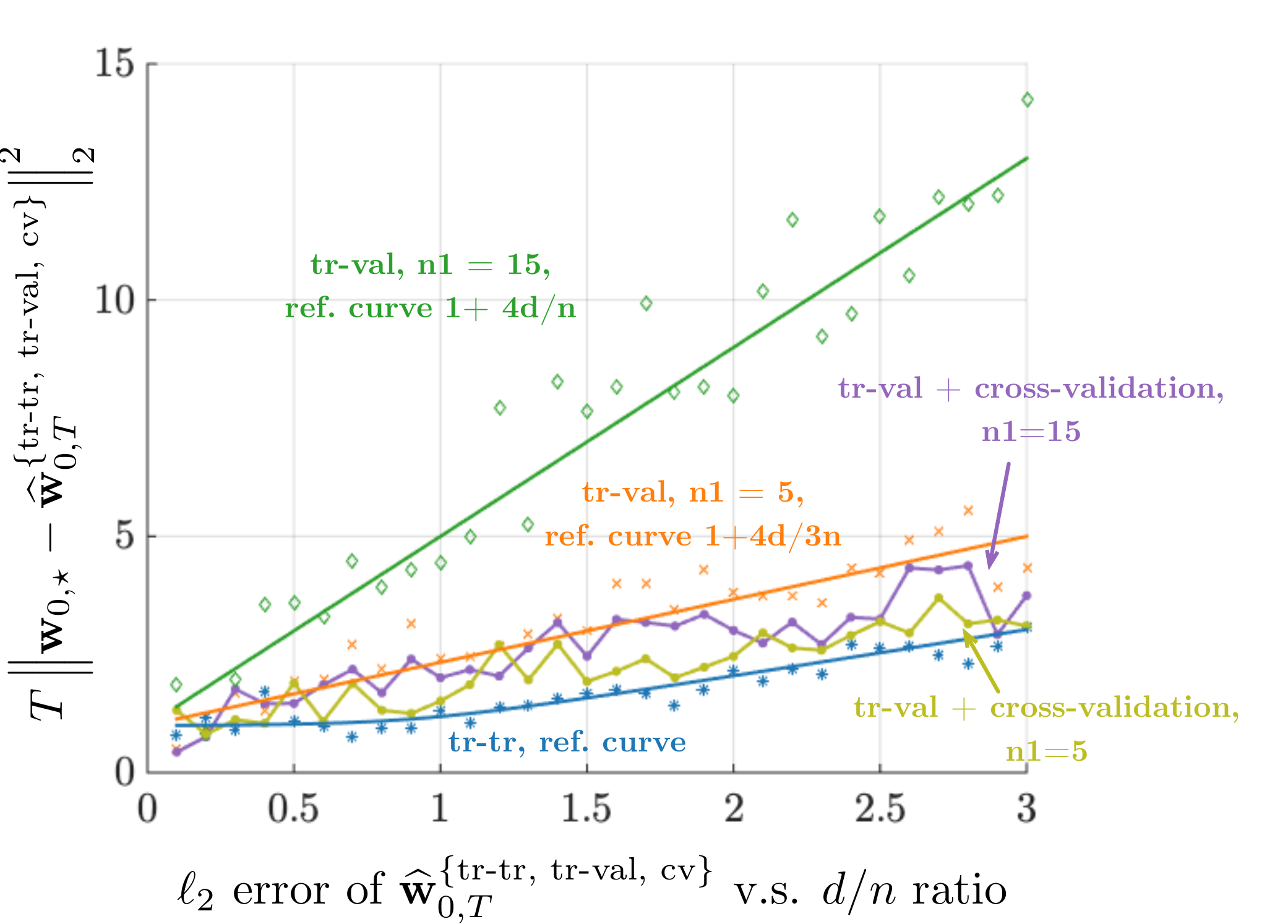}
\caption{The scaled (by $T$) $\ell_2$-error of $\hat{\wb}_{0, T}^{\{\nosp, \sp, \textrm{cv}\}}$ as the ratio $d / n$ varies from $0$ to $3$ ($n=20$ and $T = 1000$ are fixed). For the cross-validation method, the regularization coefficient $\lambda = 0.5$ is tuned.}
\label{fig:cv}
\end{figure}

\paragraph{Result}
As showin in Figure~\ref{fig:cv}, for both $(n_1,n_2)=(15,5)$ and $(n_1,n_2)=(5,15)$, using cross-validation consistently beats the performance of the \trval method. This demonstrates the variance-reduction effect of cross-validation. Note that the best performance (among the cross-validation methods) is still achieved at $n_1=5$, similar as for the vanilla \trval method. However, numerically, the best cross-validation performance is still not as good as the \trtr method.

\paragraph{Leave-one-out cross-validation}
Figure~\ref{fig:loo_cv} left further tests with an increased number of per-task samples $n = 40$, and incorporates the \trval method with the leave-one-out cross-validation, i.e., $(n_1, n_2) = (39, 1)$ and $(n_1, n_2) = (1, 39)$. We repeat the experiment $10$ times for plotting the error bar (shaded area). We see that the \trtr method still outperforms the \trval method with leave-one-out validation.

We further increase the per-task sample size $n$ to $200$, and test the leave-one-out method with a sample split of $(n_1, n_2) = (1, 199)$. We adopt a matrix inverse trick to mitigate the computational overhead of finding $\cA_\lambda(\wb_0; \Xb_t^{\train, j}, \yb_t^{\train, j})$. To ease the computation, we also vary $d$ from $0$ to $400$ on a coarse grid (with an increment of $80$). From Figure \ref{fig:loo_cv} right, we see that the leave-one-out method can slightly beat the \trtr method for some $d/n$ values. Compared to $n = 20$ and $n = 40$ experiments, this is the first time of seeing leave-one-out method outperforms the \trtr method. We suspect that the per-task sample size $n$ plays a vital role in the power of the leave-one-out method: a large $n$ tends to have a strong variance reduction effect in the leave-one-out method, so that the performance can be improved. Yet using the leave-one-out method with a large $n$ invokes a high computational burden.
\begin{figure}[h]
\centering
\includegraphics[width = 0.44\textwidth]{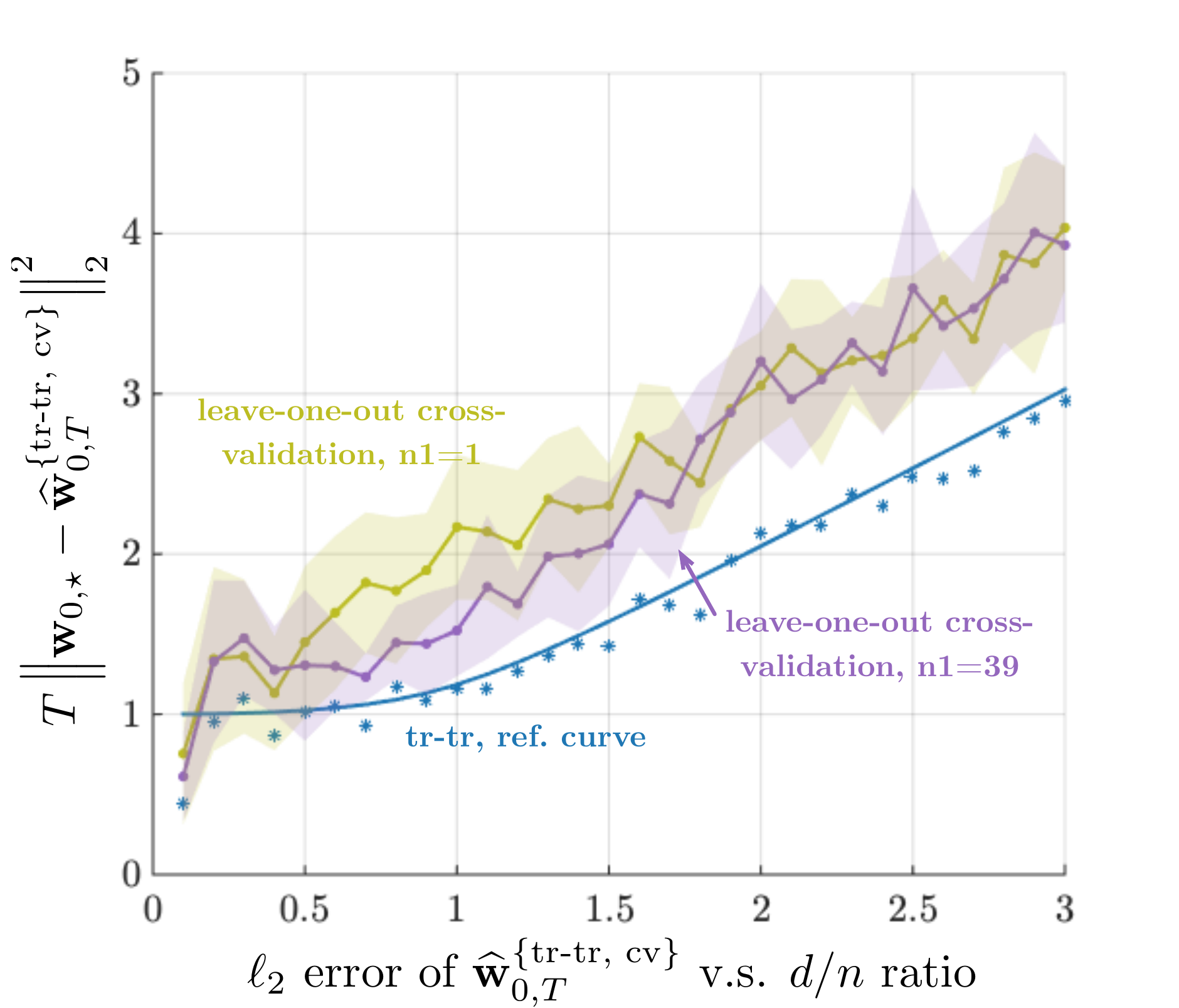}
~
\includegraphics[width = 0.48\textwidth]{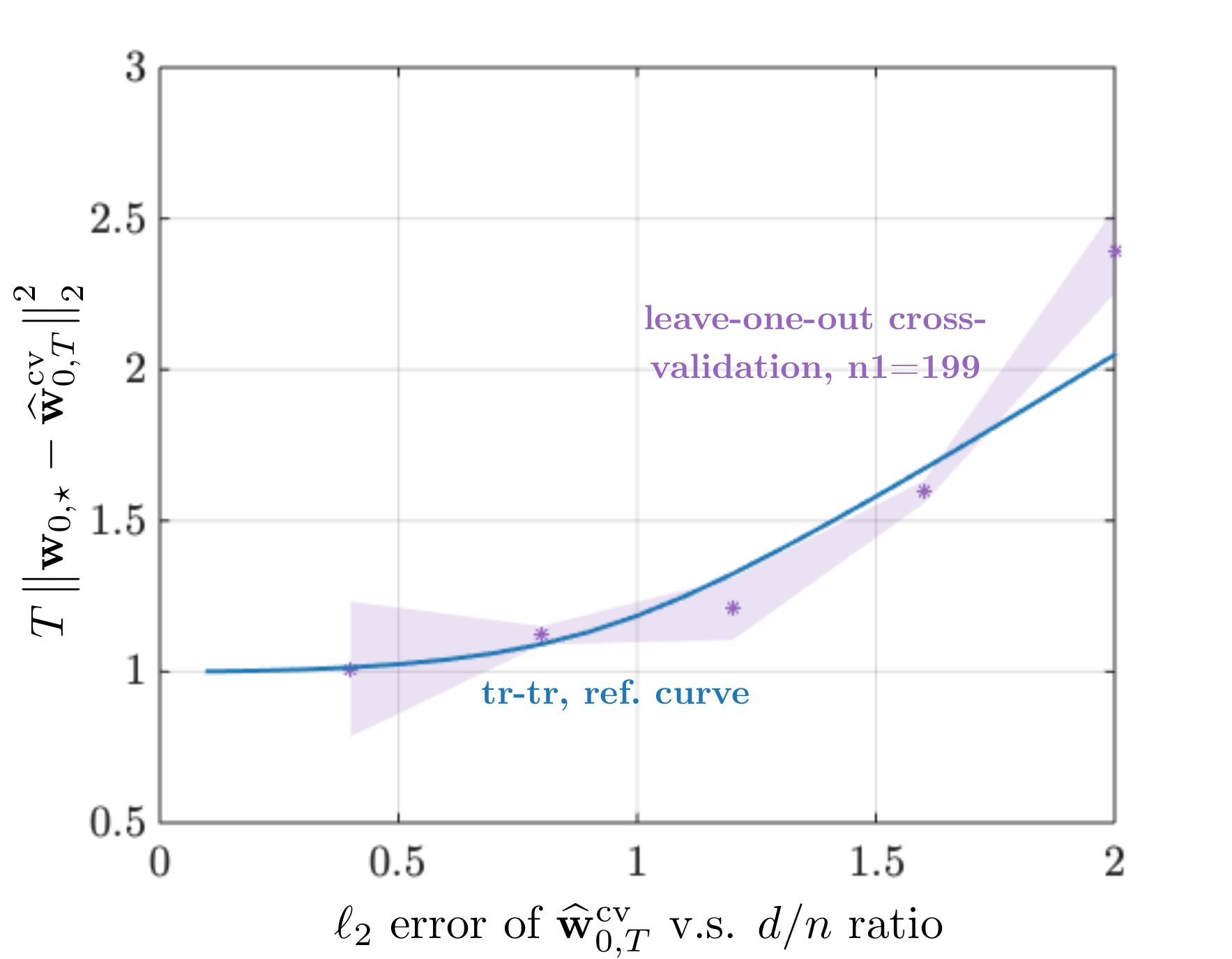}
\caption{The scaled (by $T$) $\ell_2$-error of $\hat{\wb}_{0, T}^{\{\nosp, \textrm{cv}\}}$ as the ratio $d / n$ varies from $0$ to $3$ ($n\in\set{40, 200}$ and $T = 1000$ are fixed). For the cross-validation method, the regularization coefficient $\lambda = 0.5$. Left: $n=40$. Leave-out-out CV performs worse than the \trtr method. Right: $n=200$. Leave-one-out CV appears better than the \trtr method for $d/n\in\set{1.2, 1.6}$.}
\label{fig:loo_cv}
\end{figure}

\newpage

\end{document}